%% file: main.tex
\documentclass[conference]{IEEEtran}
\usepackage{times}
\usepackage{balance}

\input{subtex/packages}

\input{subtex/commands}

\makeatletter
\def\footnoterule{\kern-3\p@
  \hrule \@width 2in \kern 2.6\p@} 
\makeatother

\usepackage{graphicx}
\usepackage{etoolbox}
\usepackage{caption} \newcommand{\insertfig}{    \includegraphics[width=\textwidth]{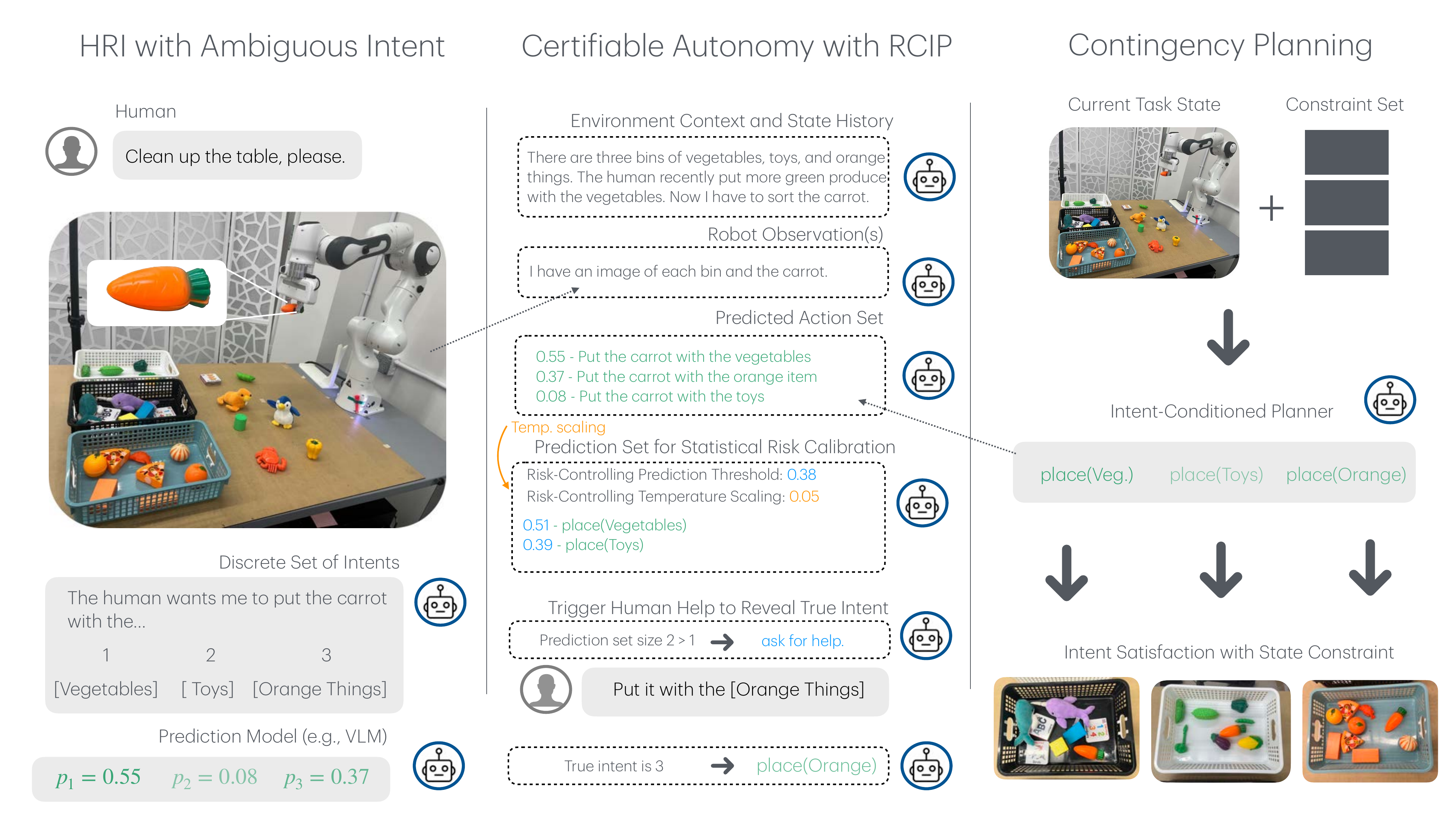}
\captionof*{figure}{Fig. 1: Risk-Calibrated Interactive Planning (RCIP) statistically calibrates risk for human-robot interaction. Given a set of possible human intents and confidence scores, a planner generates a weighted set of actions. The set of actions from each plan are collected in a set according to a threshold on the predicted intents. If there is more than one action in the set, the robot asks for help.}}

\makeatletter
\apptocmd{\@maketitle}{\centering\insertfig}{}{}
\makeatother

\newcommand{\rebuttal}[1]{#1}

\title{\rebuttal{Risk-Calibrated Human-Robot Interaction \\via Set-Valued Intent Prediction}}
\author{\large Justin Lidard, Hang Pham, Ariel Bachman, Bryan Boateng, Anirudha Majumdar\\
\small Department of Mechanical and Aerospace Engineering\\
Princeton University, Princeton, New Jersey 08540\\
Email: \texttt{jlidard@princeton.edu}}

\begin{document}

\maketitle

\begin{abstract}
        Tasks where robots must \rebuttal{anticipate human intent}, such as navigating around a cluttered home or sorting everyday items, are challenging because they exhibit a wide range of valid actions that lead to similar outcomes. Moreover, \textit{zero-shot} intent-prediction in human-robot cooperative scenarios is an especially challenging problem because it requires the robot to infer and adapt \textit{on the fly} to a latent human intent, which could vary significantly from human to human. Recently, deep learned motion prediction models have shown promising results in predicting human intent but are prone to being \textit{confidently incorrect.}  In this work, we present Risk-Calibrated Interactive Planning (RCIP), which is a framework for measuring and calibrating risk associated with uncertain action selection in human-robot cooperation, with the fundamental idea that the robot should ask for human clarification when uncertainty in the human's intent may adversely affect task performance. RCIP builds on the theory of \textit{set-valued risk calibration} to provide a finite-sample statistical guarantee on the risk incurred by the robot while minimizing the cost of human clarification in complex multi-step settings. Our main insight is to frame the risk control problem as a \textit{sequence-level} multi-hypothesis testing problem, allowing efficient calibration using a low-dimensional parameter that controls a pre-trained risk-aware policy. Experiments across a variety of simulated and real-world environments demonstrate RCIP's ability to predict and adapt to a diverse set of dynamic human intents.\footnote{Website with additional information, videos, and code: \href{https://risk-calibrated-planning.github.io/}{https://risk-calibrated-planning.github.io/}}
\end{abstract}

\input{subtex/introduction}

\input{subtex/related_work}
\input{subtex/formulation2}

\input{subtex/approach}
\input{subtex/experiments}
\input{subtex/limitations}

\input{subtex/conclusion}

\bibliographystyle{IEEEtran}
\bibliography{IEEEabrv,references.bib}

\newpage
\onecolumn
\appendix

\input{subtex/appendix}

\end{document}

%% file: subtex/packages.tex
\usepackage[utf8]{inputenc} 
\usepackage[T1]{fontenc}    
\usepackage{url}            

\usepackage{booktabs}       
\usepackage{caption}

\usepackage{amsfonts}       
\usepackage{nicefrac}       
\usepackage{microtype}      
\usepackage[dvipsnames]{xcolor}        
\usepackage{graphicx}
\usepackage{ mathrsfs }

\usepackage{amsthm}

\usepackage{amsmath}
\usepackage{mathtools}
\usepackage{tabularx}
\usepackage{framed}
\usepackage{bbm}
\usepackage{bm}
\usepackage{xspace}
\usepackage{amssymb}
\usepackage{multirow}
\usepackage{physics}
\usepackage{siunitx}
\usepackage{algorithm}
\usepackage{algpseudocode}

\usepackage{cite}

\usepackage{nccmath}

\usepackage{colortbl}

\usepackage{placeins}

\usepackage{subfigure}

\usepackage{lipsum}

\usepackage{thmtools}
\usepackage{thm-restate}

\makeatletter
\let\NAT@parse\undefined
\makeatother
\usepackage{hyperref}       
\hypersetup{
    colorlinks,
    linkcolor={blue!90!black},
    citecolor={blue!90!black},
    urlcolor={blue!95!black}
}

\usepackage[capitalise]{cleveref}       
\usepackage{crossreftools}  
\usepackage[all=normal,paragraphs=tight,floats=tight,mathspacing=tight,wordspacing=tight,tracking=tight]{savetrees}

\usepackage{etoolbox} 

%% file: subtex/commands.tex
\newtheorem*{remark}{Remark}
\usepackage{amsthm}

\newtheorem{proposition}{Proposition}    
\newtheorem{corollary}{Corollary}

\DeclareMathOperator*{\argmin}{\arg\!\min}

\DeclareBoldMathCommand\vpi{\pi}

%% file: subtex/introduction.tex
\section{Introduction}

Predicting and understanding human intent is a critical task for robotics, specifically for safe interaction with humans in cluttered, close-quarters environments. However, human intent prediction is challenging because no two humans may have the same preferences, and intents may differ depending on the specific environment. As an example, a robot is tasked with sorting items into three bins  based on an example provided by the human  (see Fig.~\hyperlink{page.1}{1}). While the bins have a ground-truth sorting criterion known by the human (vegetables, children's toys, and miscellaneous orange items), the robot must infer the human's intent in order to sort new items. Given the provided context, the robot should be able to sort some unambiguous items (e.g., the crab) autonomously, while other items (e.g., the carrot) may be placed into multiple bins, resulting in \textit{situational ambiguity}. If asked to operate fully autonomously, the robot must take a \textit{risk} and guess the correct placement for the carrot. However, the robot may also \textit{ask for help} if it is unsure, guaranteeing the correct action but potentially burdening the human.
In this work, we study the tradeoff between risk and autonomy governing optimal action selection in the face of situational ambiguity. 


Recently, calibrated predict-then-plan (also known as contingency planning) \cite{fridovich2020confidence, lindemann2023safe} approaches have demonstrated the ability to generate provably safe plans by first using confidence-aware prediction models to generate a set of possible futures and then constructing a safe plan that accommodates for the future uncertainty. These approaches enable synthesis of large amounts of scene-specific context (such as image or map information) while simultaneously providing a guarantee on the plan success rate by calibrating the coverage of the prediction. However, one of the major challenges of predict-then-plan approaches comes from \textit{multi-modal human behavior}: if the distribution of human actions contains multiple high-level behaviors, a single robot plan may become overly conservative in trying to accommodate all possible human intents. Moreover, environments themselves may generate additional sources of ambiguity that may result in unsafe behavior from the robot if misinterpreted. In such cases, if possible, the robot should ask for help in order to clarify the human's intent instead of committing to a potentially unsafe action.

Our approach utilizes deep-learned human intent prediction models (e.g., \cite{ salzmann2020trajectron++, achiam2023gpt}) for understanding interactivity, and rigorously quantifies the uncertainty of these models in order to decide when to ask for help. As shown in Fig.~\hyperlink{page.1}{1} (middle), we produce a limited set of human intents based on the prediction model's confidence scores. For each predicted intent, we plan a sequence of actions that satisfy an environment objective, such as placing the item in the correct bin. \rebuttal{Depending on the robot's confidence level and the human's preferred level of autonomy, the robot can either take a risk or ask for help. To allow the human to specific different levels of robot autonomy (more or less confident predictions), we assume that the predictor has a small number of tunable model parameters (such as the  temperature used in softmax scoring).} We use a small calibration dataset of human-robot interactions to choose a set of valid parameters that provide a level of risk and autonomy set in advance by the user. By leveraging recent advances in distribution-free risk control \cite{angelopoulos2021learn}, we show that the robot's behavior can simultaneously limit several notions of risk. We formalize this challenge via two objectives: (i) \textit{statistical risk calibration (SRC)}: the robot should seek sufficient help from the human when necessary to ensure a statistically guaranteed level risk specified by the user, and (ii) \textit{flexible autonomy}: the robot should ask for a minimal amount of help as specified by the user by narrowing down situational ambiguities through planning. We refer to these simultaneous objectives, with help from the human when necessary, as Risk-Calibrated Interactive Planning (RCIP). 

\noindent 
\textbf{Statement of contributions.} In this work, we introduce RCIP, a framework for measuring and calibrating risk in situations that involve interactions with humans with potentially ambiguous action choices. By reasoning about the human's desired task outcome in the space of \textit{intents}, we efficiently plan safe actions in the face of diverse, multi-modal human behavior, and ask for help when necessary. We make the following contributions: \textbf{(1)} We demonstrate how to use SRC to control the planning error rate across a set of model hyper-parameters, allowing flexible but provably safe levels of autonomy.  \textbf{(2)} We prove theoretical guarantees for multi-dimensional risk control for both single-step and multi-step planning problems: with a set of $K$ user-specified risk budgets $(\alpha_1, ..., \alpha_K)$  for different measures of risk (e.g., probability of failure and probability that the robot asks for help) and the robot performs the task correctly (with high probability) by asking for help if any of the $K$ risk budgets will be violated. \textbf{(3)} We evaluate RCIP in both simulation and hardware with a suite of human-robot interactive planning tasks with various styles of situational ambiguity (spatial, contextual, semantic). Experiments across multiple platforms and human uncertainty showcase the ability of RCIP to provide statistically guaranteed task success rates while providing more flexible autonomy levels than baseline approaches. RCIP reduces the amount of human help by $5-30\%$ versus baseline approaches.


%% file: subtex/related_work.tex
\section{Related Work}

RCIP brings together techniques from contingency planning, human intent prediction, and conformal prediction and empirical risk control. We discuss related work in each area below.

\subsection{Contingency Planning and Priviledged Learning}



Contingency planning \cite{hardy2013contingency} is a growing literature on planning for multi-agent interactive scenarios where future outcomes are diverse. Recent approaches \cite{zhan2016non, chen2022interactive, nair2022stochastic, Cui2021-zf} typically favor a predict-then-plan approach, wherein multi-modal motion predictions are first generated and then used to produce a set of safe plans conditioned on each prediction. The authors of \cite{peters2023contingency} formulate a multi-agent contingency planning problem as a generalized Nash equilibrium problem, thereby assuming that agents are non-cooperative. In this work, we assume that the human and robot act in good faith (i.e., they are cooperative).  

Similar to contingency planning is the \textit{learning under privileged information} paradigm \cite{vapnik2009new, pechyony2010theory, chen2020learning},  which provides the learning algorithm with additional information during training to help bootstrap near-optimal behaviors. Privileged learning has shown empirical success in semantic reasoning \cite{sharmanska2013learning}, vision-based robotic manipulation \cite{james2022q}, and learning policies that can be deployed in the wild \cite{lee2020learning, miki2022learning, loquercio2021learning}.  In \cite{bajcsy2023learning}, privileged information about the human's trajectory is used to train a policy that most efficiently apprehends a human opponent, and a partially-observed deployment policy is distilled using a teacher-student paradigm. Similarly, in \cite{monaci2022dipcan}, a visuomotor policy for social navigation is trained by using exact pedestrian positions during training, and a model for estimating for the position embedding is distilled from the privileged embedding. 

In this work, we provide the robot with additional information about the internal state of the \rebuttal{human} during the planning phase (for example, their preference for a particular bin as shown in Fig. ~\hyperlink{page.1}{1}). We eliminate the need for a separate distillation procedure by instead using a set-valued prediction strategy, introduced in the following sections. We use contingency planning (and privileged learning) to find (or learn) \rebuttal{an intent-conditioned policy}, which can then be used to predict an optimal action via an upstream predictor. By allowing the robot to ask for help when it is uncertain, we statistically quantify risk associated with the robot acting optimally, even when it is uncertain. \rebuttal{ In practice, we show that intent-condition planning permits a modular approach to planning in uncertain environments: the intent predictor allows the robot to discern high-level uncertainly in human preferences for task completion, while a low-level intent-conditioned planner allows the robot to ensures that an acceptable intent-conditioned policy exists.}
 
\subsection{Human Intent Prediction}

Predicting intent of humans for downstream planning has been widely applied in autonomous driving \cite{shi2022motion, huang2022hyper, zhou2022hivt}, social navigation \cite{liu2023intention, agand2022human, salzmann2020trajectron++}, and game theory \cite{meta2022human}. Several works \cite{huang2022hyper, gu2021densetnt} use a discrete latent variable to capture qualitative behaviors in human motion. To aid in human goal satisfaction, the authors of \cite{he2023learning} show that human actions can be predicted directly in interactive settings, \rebuttal{but the prediction model must be trained on a set of representative interactions. To showcase RCIP's ability to enable \textit{zero-shot} intent-conditioned planning we leverage recent advances in vision language models (VLMs) \cite{radford2021learning} to predict human intent by extracting task-relevant semantic features through vision, \textit{without fine-tuning on specific interactions}.}

\rebuttal{In this work, we show that both task-specific (e.g., \cite{salzmann2020trajectron++}) and zero-shot 
intent predictors (e.g., \cite{achiam2023gpt}) can be used to achieve a variety of robot autonomy levels when combined with RCIP. We enable efficient transfer to downstream planning by restricting the support of the predictor to a small set of human intents, allowing the model to focus on high-level uncertainties.}

\subsection{Conformal Prediction and Empirical Risk Control}

Conformal prediction \cite{vovk2016criteria, vovk2012conditional, sadinle2019least} has recently gained popularity in a variety of machine learning and robotics  applications due to its ability to rigorously quantify and calibrate uncertainty. A recent line of works \cite{stankeviciute2021conformal, strawn2023conformal, ren2023robots} has extended the theory from prediction of labels (e.g. actions) to sequences (e.g. trajectories). Several works \cite{dixit2023adaptive, gibbs2021adaptive} have studied \textit{adaptive} conformal prediction, wherein a robot's predictive conservativeness is dynamically adjusted within a policy rollout by assuming that there always exists a conservative fallback policy. Finally, some recent works \cite{ bates2021distribution, lekeufack2023conformal} have extended conformal prediction theory to handle more general notions of risk. 

Our work differs in three key ways: (i) we provide a separate calibration stage in which the robot can adjust its parameterization of prediction sets through a modest-size dataset of interactive scenarios, reducing the number of ``unrecoverable'' scenarios in which the robot exceeds its risk budget early on in a rollout, and (ii) we provide a way to synthesize from scratch risk-averse control policies, and (iii) we reason about human uncertainty in the space of \textit{intents}, permitting a more natural way to capture diverse interactive behaviors than other representations (e.g. trajectories). 

%% file: subtex/formulation2.tex
\section{Problem Formulation} \label{formulation}

In this section, we pose the problem of human-robot cooperation with intent uncertainty as a partially observable Markov decision process (POMDP). We present a brief overview of the prediction-to-action pipeline and our goals of risk specification and flexible autonomy. 

\subsection{Dynamic Programming with Intent Uncertainty}

\begin{figure*}
    \centering
    \includegraphics[trim={0 6cm 0 0},width=\textwidth]{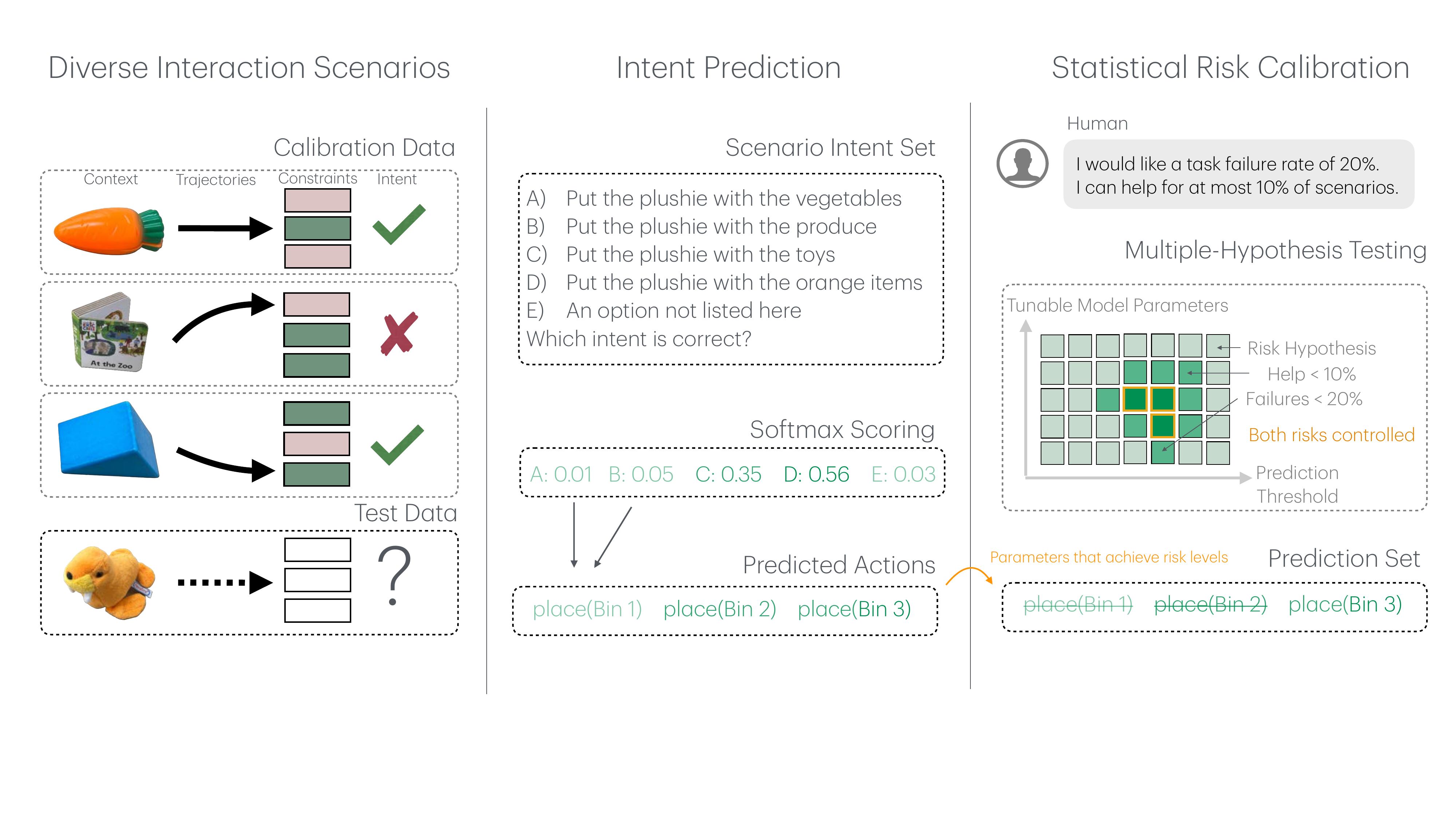}
    \caption{RCIP formulates interactive planning as a multi-hypothesis risk control problem. Using a small set of calibration scenarios, RCIP computes step-wise prediction losses to form an aggregate emperical risk estimate. Using a risk limit, for each pair $(\lambda, \theta)$ of prediction thresholds and tunable model parameters, RCIP evaluates the hypothesis that the test set risk is above the limit. Thus, for all hypotheses that are rejected, the test set risk satisfies the threshold (with high probability). }
    \label{fig:pipeline}
\end{figure*}

\vspace{0.10cm}\noindent 
\textbf{Environment Dynamics.} We consider an interaction between a robot $R$ and human $H$ in environment $e$, governed by a nonlinear dynamical system with time horizon $T$:
\begin{equation} \label{eqn: dynamics}
   \rebuttal {x_{t+1} = f_t(x_t, u_t;\ e)} \quad \forall t \in [T],
\end{equation}
where $x_t \in \mathcal S \subseteq \mathbb R^n$ is the joint state of the system and  $u_t \in \mathcal U \subseteq \mathbb R^m$ is the joint (robot-human) control input $(u_t^{R} , \ u_t^{H})$. We use $[T]$ notation to denote the set of natural numbers up to $T$. We use a superscript for individual agent indexing, and we use bar notation to denote aggregation over time, e.g. $\bar{x}_t = (x_1,\ \dotsc, \ x_t)$. Let $ \pi = (\pi^R, \ \pi^H)$ denote the joint policy governing system \eqref{eqn: dynamics}. We permit the human's action to be drawn from a potentially multi-modal distribution $\pi^H$. \rebuttal{We model the modes of $\pi^H$ as corresponding to different human intents, which may be time-varying, and their dynamics is described next.}


\vspace{0.10cm}\noindent 
\textbf{Intent Dynamics.} We assume that the human's (potentially unknown) policy $\pi^H$ is parameterized by a discrete latent variable with the following dynamics:
\begin{equation} \label{eqn: z_dynamics}
    \rebuttal {z_{t+1} \sim q_t(\cdot| x_t, z_t;\ H) \quad \forall t \in [T]},
\end{equation}
where $z_t \in \mathcal Z = [N]$ characterizes the human's intent at time $t$, and $N$ is the number of high-level human behaviors. We assume that conditioned on the human's latent intent (which may be stochastic), \rebuttal{ each agent's action is drawn from a (potentially time-varying) probability distribution, i.e., $u_t^i = \pi^{i}_t(\cdot | x_t, z_t)$, for $i \in \{R, \ H\}$. }


\vspace{0.10cm}\noindent 
\textbf{Planning Objective.} Each agent $i\in\{R,\ H\}$ has the goal to minimize their corresponding cost function $J^i$ in finite-horizon $T$ with running cost $l^i$. The cumulative cost of a policy $\pi^i$ starting from initial state $x$ and a \textit{known} human intent $z$ is
\begin{equation} \label{eqn: cumulative_cost}
    J^i(x, z, \pi^i) = \mathbb E^{\pi, q} \Bigg[ \sum_{t=1}^{T} \; \; l^i(x_t, u_t) \Big| x_1 = x, z_1 = z\Bigg].
\end{equation}
The objective of agent $i$ is to find a policy $\pi^i$ that minimizes eqn.~\eqref{eqn: cumulative_cost}. To ensure the safety of the human, we add an additional set of inequality constraints $h^i$ that depend on the (time-varying) intent of the human:
\begin{equation} \label{eqn: cumulative_loss_constrained}
\begin{array}{rl}
    \displaystyle \min_{\pi^i} & J^i(x, z, \pi^i) \\
    \textrm{s.t.} & h^i(x_t, z_t) \leq 0 \quad \forall t \in [T]\\
    & (x_1, z_1) = (x, z).
\end{array}
\end{equation}
\vspace{0.10cm}\noindent
\textbf{\rebuttal{Intent-Conditioned Action Selection.}} For each intent $z$, the value function associated with the latent intent $z$ can be evaluated as a function of the state-intent pair $(x, z)$. That is,
\begin{equation} \label{eqn: value_function_general}
        V^i(x, z) = \inf_{\pi^i} \Big\{J^i(x, z, \pi^i)\Big\}.
\end{equation}
The Bellman optimality principle states that the optimal policy satisfies Bellman's equation:
\begin{equation} \label{eqn: value_function}
        \rebuttal{V^i(x, z) = \inf_{\pi^i} \Big\{\mathbb E^{(u, z^\prime) \sim (\pi, q)} \big[l^i(x, u) + V^{i}(f(x, u;\; e), z^\prime) \big] \Big\}},
\end{equation}
where $z_{t+1}$ is sampled according to eqn.~\eqref{eqn: z_dynamics}.
The action-value function is similarly defined as 
\begin{equation} \label{eqn: action_value_function}
    \rebuttal{Q^{i}(x, z, u^i) = l^i(x, u)+ \mathbb E^{(u, z^\prime) \sim (\pi, q)}[V^i(f(x, u; \;e), z^\prime)].}
\end{equation}
\rebuttal{Eqn.~\eqref{eqn: action_value_function} states that since both agents' policies are conditioned on the human's intent (interpreted as a state constraint), the optimal action can be approximated with a variety of existing policy optimization or dynamic programming techniques, e.g., PPO \cite{schulman2017proximal}, A-star graph search \cite{hart1968formal}, or Rapidly Exploring Random Trees (RRT) \cite{lavalle2001randomized}. In particular, the optimal action can be approximated for both agents as the minimizer of the $Q$-function, i.e.,}
\begin{equation} \label{eqn: conditional action selection}
    u^{i*}(z) = \argmin_u \; \; Q^i (x, z, u) \quad \forall i \in \{R, \ H\}.
\end{equation}
\rebuttal{We emphasize that, in reality, the robot need not \textit{exactly} actuate the optimal action as long as the intent constraint in \eqref{eqn: cumulative_loss_constrained} is satisfied.}

\vspace{0.10cm}\noindent
\textbf{Intent Prediction.} During deployment, the true human intent $z$ is not observed. However, we assume access to a frozen model $g_\theta$ with a small number of tunable parameters $\theta$ that \textit{predicts} human intent based on a sequence $\bar x_t$ of prior states. \rebuttal{Specifically, $g_\theta(\bar x_t, z)$ seeks to estimate the probability mass $q(z | x_{t-1}, z_{t-1}) $ of each intent $z \in \mathcal Z$ by producing a set of softmax confidence scores on the simplex $\Delta^{|\mathcal Z|-1}$. That is,}
\begin{equation} \label{eqn: intent_softmax}
    \rebuttal{g_\theta(\bar x_t, z) = \frac{ \exp(G_\theta(\bar x_t, z))}{\sum_{z^\prime \in \mathcal Z} \exp(G_\theta(\bar x_t, z^\prime))} \quad \forall z \in [\mathcal Z],}
\end{equation}
\rebuttal{where $G_\theta$ outputs the logit scores for each intent. In general, $g_\theta$ will output heuristic and \textit{uncalibrated} confidence scores, and the approximation of the true intent dynamics $q$ may be poor, since $g_\theta$ need not be trained on the same human intent dynamics as given by \eqref{eqn: z_dynamics}. However, we show in the following analysis (c.f. Section \ref{section: approach}) that by asking for help, the robot is able to overcome incorrect predictions through strategic selection of human interventions with SRC.}

\subsection{Risk-Calibrated Interactive Planning}
\label{sec:RCIP overview}

\rebuttal{RCIP seeks to distill a weighted set of actions using the scores from $g_\theta$ and the actions from the intent-conditioned planner. By using an action \textit{set} during action prediction, RCIP ensures that a plan exists for each possible human intent. We now describe how the action set is obtained.} 

\vspace{0.10cm}\noindent
\textbf{Predicted Action Set.} We first aggregate confidence scores from $g_\theta$ into an intermediate set $\mathcal S_{\lambda, \theta}  \subseteq \mathcal Z$ of predicted intents via the rule
\begin{equation} \label{Eqn: predicted_intent_set}
    \mathcal S_{\lambda, \theta}(\bar x_t) = \{ z \in \mathcal Z: g_\theta(\bar x_t, z) \geq \lambda \},
\end{equation}
where $\lambda$ is a confidence threshold (cf. Section \ref{section: approach}). Since human intent uncertainty alone may not alter the optimal robot plan, we compute a set of predicted actions from the set of predicted intents as 
\begin{equation} \label{Eqn: predicted_action_set}
    \mathcal T_{\lambda, \theta}(\bar x_t) = \{ u \in \mathcal U: \exists z \in  \mathcal S_{\lambda, \theta}  \text{ s.t. } u = u^{R*}(z) \text{ and } g_\theta^*(\bar x_t, z)  \geq \lambda \},
\end{equation}
where we define $g_\theta^*$ as the sum of all intent-based confidence scores that lead to the same action, i.e.,  $g_\theta^*(\bar x_t, z):= \sum_{z^\prime \in \mathcal Z} g_\theta(\bar x_t, z^\prime) \mathbbm{1}\{u^{R*}(z)=u^{R*}(z^\prime)\}$. 
\begin{remark}
    \rebuttal{
    Eqns. \eqref{Eqn: predicted_intent_set} and \eqref{Eqn: predicted_action_set} convey the fact that the mapping from intents to actions can be many-to-one. Taking Fig.~\hyperlink{page.1}{1} as an example, if the predicted intent set is augmented to include \{Vegetables, Produce, Toys, Orange\}, the predicted intent set will have $|\mathcal S_{\lambda, \theta}|=4$ intents for sorting, but the predicted action set will only have $|\mathcal T_{\lambda, \theta}|=3$ actions, since Vegetables and Produce are the same bin and the action is the same. Quantifying uncertainty in action space allows RCIP to ask for help only when the predicted actions are different.}
\end{remark}

\vspace{0.10cm}\noindent
\textbf{Policy Deployment.} We now define our overall robot policy $\pi^R$. Given the predicted action set $\mathcal T_{\lambda, \theta}(\bar x_t)$ defined in Eqn.~\eqref{Eqn: predicted_action_set}, the robot has two behaviors:
\begin{enumerate}
    \item \textbf{Autonomy.} If $\mathcal T_{\lambda, \theta}(\bar x_t)$ is a singleton, then the robot is confident in the predicted action, and the action is executed. 
    
    \item \textbf{Triggering Help.} If $\mathcal T_{\lambda, \theta}(\bar x_t)$ is not a singleton, then the robot triggers human help, and the human reveals their true intent, $z^*$. The robot executes the action $u^{R*}(z^*)$. 
\end{enumerate}
 If $\lambda$ and $\theta$ are chosen such that $\mathcal T_{\lambda, \theta}(\bar x_t)$ is empty, the robot ceases operation and the task is failed.

\subsection{Goal: Certifiable Autonomy}

Situational ambiguity results in many potentially correct robot actions arising with potentially no safe external resolution, save for direct human intervention (see e.g. Fig.~\hyperlink{page.1}{1}). Our goal in this work is to address certifiable autonomy: selecting a set of model parameters $(\theta, \lambda)$ that achieve multiple user-specified levels of risk. \rebuttal{As shown in Fig.~\ref{fig:pipeline}, we formalize this problem by considering a joint distribution $\mathcal D$ over scenarios $\mathcal \xi := (e, H, l, h)$, where $e$ is an environment (a POMDP) with dynamics $f$ are parameterized by $e$,  $H$ is a particular human with stochastic intent dynamics $q$, and $l$ and $h$ are respectively a cost and constraint function encoding the robot's goal, which is assumed to be known \textit{a-priori} by both the human and robot}. We do not assume knowledge of $\mathcal D$ except for the availability of a modestly-sized calibration dataset \rebuttal{$D_{\text{cal}}$} containing $400$ samples from $\mathcal D$. We formalize certifiable autonomy in our context as (i) risk calibration: the robot must meet a set of user-specified risk levels $(\alpha_1,\dotsc, \alpha_K)$ with user-specified probability over new scenarios $\xi \sim \mathcal D$, and (ii) flexible autonomy: the policy $\Pi^R$ should return a set of model parameters that control each risk but allow different high-level behaviors.

%% file: subtex/approach.tex
\section{Approach} \label{section: approach}

In this section we present a procedure for guaranteeing optimal action selection while controlling a user-specified notion of risk. We introduce statistical risk calibration below, then present the different 
practical settings we consider (single-step, multi-step, and multi-risk). 

\subsection{Background: Statistical Risk Calibration}

\vspace{0.10cm}\noindent
\textbf{What is a risk?} We now present an approach for controlling the \textit{risk} of the robot's multi-modal policy \rebuttal{$\pi^R$} by calibrating when the robot should ask for help at inference time according to a user-specified notion of risk. 
Our approach builds on the Learn-then-Test framework for distribution-free statistical risk control \cite{angelopoulos2021learn}. Let $\mathcal D$ be an unknown distribution over i.i.d. scenarios such that $\xi \sim \mathcal D$. If we fix a policy for the human and the robot and assume that the robot has ground-truth knowledge of z, then the distribution over scenarios induces a distribution over the context-label pairs $(\bar{x}, z)$, where the context  contains a history of the previous states up to and including the current state at time $t$. 

Assume that we are given a risk signal $R$ that we wish to control, where $R \in [0, 1]$ measures an expected loss as a function of the prediction threshold $\lambda$ and model parameters $\theta$. 
For fixed parameters $(\lambda, \theta)$, the expected loss is itself a function of the context, prediction sets, and true labels over the unknown distribution $\mathcal{D}$, i.e.
\begin{equation} \label{eqn: risk_definition_expectation}
        R(\lambda, \theta) =  \mathbb E^{(\bar{x}, z) \sim \mathcal D} \Big[ L(\bar{x}, \mathcal T_{\lambda, \theta} (\bar{{x}}), z) \Big],
\end{equation}
and the loss $L$ is similarly defined on $[0, 1]$. 

\vspace{0.10cm}\noindent
\textbf{Bounding the Probability of Suboptimal Actions.} As an example, $L$ could be miscoverage, i.e., $L=0$ if the optimal action is in the prediction set, and $L=1$ otherwise. In expectation over $\mathcal D$, the risk associated with suboptimal action selection is the \textit{miscoverage risk}, i.e
\begin{equation} \label{eqn: risk}
    R_\text{cov}(\lambda, \theta) =  \mathbb P^{(\bar x, z) \sim \mathcal D}\big(u^{R*}(z^*) \notin \mathcal T_{\lambda, \theta} (\bar{{x}}) \big).
\end{equation}

\begin{remark}
    Equation \eqref{eqn: risk} is identical to the typical conformal prediction (CP) setting \cite{vovk2005algorithmic, angelopoulos2021gentle} in which the risk is miscoverage of the true label in the prediction set. However, the standard CP framework only allows one to choose $\lambda$ to bound the miscoverage rate. In contrast, the formulation we consider permits the modification of other model parameters $\theta$ (e.g., the temperature of the softmax outputs) in order to have more fine-grained control of the prediction sets and bound risks beyond miscoverage. We will demonstrate the benefits of this flexibility empirically in later sections. 
\end{remark}

\vspace{0.10cm}\noindent
\textbf{Calibrating the Predicted Action Set.} We will assume access to a calibration set $D_{\text{cal}} = \{(\bar{x}_i, z_i)\}_{i=1}^M$ of i.i.d. random variables drawn from $\mathcal D$, which we will use to estimate the risks. We seek to take the (uncalibrated) prediction model $g_\theta^*: \mathcal X^t \times \mathcal Z \rightarrow \mathcal [0, 1]$ that produces softmax scores for each intent-conditioned action $u^{R*}(z)$. As described in Sec.~\ref{sec:RCIP overview}, we post-process the raw model outputs in $[0, 1]$ to generate a prediction set $\mathcal T_{\lambda, \theta}(\bar{x})$ containing actions; this set is parameterized by a low-dimensional set of parameters $\lambda \in \Lambda$ and $\theta \in \Theta$, where $\Lambda$ is a finite set of prediction thresholds values one wishes to test, and $\Theta$ is a finite set of model parameters, such as temperature. \rebuttal{Then, we use the calibration set in order to choose the parameter pair $(\lambda, \theta)\in \Lambda \times \Theta$ to control a user-specified risk, regardless of the quality of the predictor $g_\theta^*$. In future sections, we let $\Phi$ denote this product search space.} 

Ahead of calibration, we set a desired risk threshold $\alpha$. Our goal is to identify a set $\hat \Phi \subseteq \Phi$ such that for any $(\lambda, \theta) \in \hat \Phi$, $R(\lambda, \theta) > \alpha$ (i.e., the risk is not controlled) with some user-defined probability $\delta$. In particular, the probability $\delta$ is with respect to the randomness in sampling over the calibration dataset $D_{\text{cal}}$, which itself is randomly sampled from the unknown distribution $\mathcal D$.

\vspace{0.10cm}\noindent
\textbf{Multi-Hypothesis Testing for a Single Risk.} Since the prediction set $\mathcal T_{\lambda, \theta}$ is controlled by low-dimensional hyperparamters $
(\lambda, \theta)$ drawn from the set $\hat \Phi$, controlling a single risk is a multiple-hypothesis testing problem \cite{angelopoulos2021learn}. For each $j \in \{1, ..., |\Phi|\}$, we consider the hypothesis $\mathcal H^j$ such that the risk $R(\lambda^j, \theta^j)$ is not controlled, where $\lambda^j, \theta^j \in \Phi$. Therefore, rejecting $\mathcal H^j$ is equivalent to certifying that the risk is controlled. 
For a calibration set size $M$, define the empirical risk estimate on the calibration set:
\begin{equation} \label{eqn: risk_sample}
    \hat R^j = \frac{1}{M} \sum_{i=1}^M \; L(\bar {x}_i, \mathcal T_{\lambda^j, \theta^j}(\bar {x}_i), z_i).
\end{equation}

Using $\hat{R}^j$, the Hoeffding-Bentkus inequality \cite{bates2021distribution} gives the $j$th $p$-value as
\begin{equation} \label{eqn: HB_pval}
    p^j = \min\Big(\exp\big(-Mh_1(\max(\hat R^j, \alpha), \alpha)\big), e\hat \Phi_{\alpha, n}^\textnormal{Bin}(\lceil n \hat R^j\rceil) \Big),
\end{equation}
where $h_1(a, b) = a\log(a/b) + (1-a)  \log((1-a)/(1-b))$ and $\hat \Phi_{\alpha, n}^\textnormal{Bin}(y)$ is the cumulative distribution function of the binomial distribution with parameter $\alpha$ and number of trials $n$. 

We now have left to construct our set $\hat \Phi$ of low-dimensional parameters $(\lambda^j, \theta^j)$ that reject $\mathcal H^j$ and control the risk $R$. Bounding \eqref{eqn: risk_definition_expectation} for all $(\lambda, \theta) \in \hat \Phi$ requires that the $p$-values hold simultaneously; any nontrivial subset $\hat \Phi \subseteq \Phi$ that controls the risk is said to control the family-wise error rate (FWER). 
A simple but powerful approach, which we use in the following analysis, is to apply a union bound over a coarse grid $\mathcal J$ of initializations (e.g. each item in $\mathcal J$ is an equally spaced grid of indices of $\Phi$ ) in an iterative procedure called \textit{fixed-sequence testing} \cite{bauer1991multiple, angelopoulos2021learn}. In fixed-sequence testing, for each $j \in \mathcal J$, the set $\hat \Phi$ of valid prediction thresholds is initialized as the empty set and grown according to the rule
\begin{equation}
    \hat \Phi \leftarrow
    \begin{cases}
        \hat \Phi \cup \{(\lambda^l, \theta^l)\} \quad & (\lambda^l, \theta^l) \notin \hat \Phi \text{ and }p^l \leq \delta/|\mathcal J|, \\
        \hat \Phi & o.w.
    \end{cases}
\end{equation}
for $l \geq j$ and $j \in [\mathcal J]$. That is, parameters $(\lambda^j, \theta^j)$ are only added to $\hat \Phi$ if $\mathcal H^j$ is rejected, eliminating the need for a union bound over a large set of parameters. The set of parameters that satisfy the risk bound is given by 
\begin{equation} \label{eqn: lambda_fixed_sequence}
    \hat \Phi := \{ \lambda^j, \theta^j: p^j \leq \delta/|\mathcal J| \}.
\end{equation}
Thus, 
\begin{equation} \label{eqn: risk_bound}
    \mathbb P^{(\bar x, z) \sim \mathcal \mathcal{D}^M} \Bigg(\sup_{(\lambda, \theta) \in \hat \Phi} \big \{ R(\lambda, \theta) \big \} \leq \alpha  \Bigg) \geq 1 - \delta,
\end{equation}
where the supremum over the empty set is defined as $- \infty$. 
The calibration procedure thus yields $\hat \Phi$, which is a set of values $\lambda^j, \theta^j$ that each control the risk $R(\lambda^j, \theta^j)$ to the desired level $\alpha$ (with probability $1-\delta$ over the randomness in the calibration dataset).


\subsection{Single-Step, Single-Risk Control}


We now state our first proposition, which bounds the action miscoverage rate for single-step settings.

\begin{proposition}
    Consider a single-step setting ($T=1$) where we use risk calibration parameters $(\lambda, \theta) \in \hat \Phi$ to generate predicted action sets and seek help whenever the prediction set is not a singleton (cf. Sec.~\ref{sec:RCIP overview}). If the FWER-controlling parameter set $\hat \Phi$ is non-empty, then with probability $1-\delta$ over the sampling of the calibration set, the new scenarios drawn from $\mathcal D$ incur at most $\alpha_1$ rate of optimal action miscoverage. 
\end{proposition}
\begin{proof}
    The proof follows immediately from application of fixed-sequence testing to the $p$-values obtained from the Hoeffding-Bentkus inequality, as given in \cite{angelopoulos2021learn}, and is nearly identical to the conformal prediction setting \cite{ren2023robots} up to the choice of concentration inequality.
\end{proof} 

\subsection{Single-Step, Multi-Risk Control} We now introduce two key risks that will play a significant role in determining the robot's level of autonomy. The first relates to suboptimal action selection and is defined in Eqn.~\eqref{eqn: risk}, and the second relates to the level of human help. 

While typical conformal prediction guarantees a minimal \textit{average} prediction set size, we are interested in minimizing the human help rate, introduced here.

\vspace{0.10cm}\noindent
\textbf{Bounding the Human Help Rate.} We now seek to provide an additional bound on the probability of asking for human help, i.e., 
\begin{equation} \label{eqn: human help rate}
    R_\text{help}(\lambda, \theta) =  \mathbb P^{(\bar x, z)\sim \mathcal D}\big(|\mathcal T_{\lambda, \theta} (\bar{{x}})| > 1 \big).
\end{equation}
Eqn.~\eqref{eqn: human help rate} is the fraction of scenarios where the prediction set is not a singleton, which is exactly the fraction of scenarios where help is needed. Hence, optimizing for action miscoverage alone may result in the robot asking for help an excessive amount of times and over-burdening the human. Instead, we apply the risk control procedure again to the help-rate risk. As before, define risk thresholds $\alpha_1$ and $\alpha_2$ and null hypotheses 
\begin{equation} \label{eqn: max_p_val}
    \mathcal H^{j}_k: \quad R_k(\lambda^j, \theta^j) \geq \alpha_k \quad k \in \{\textnormal{cov}, \ \textnormal{help}\}
\end{equation}
for $j \in [|\hat \Phi|]$. We now present a bound on the probability that both risks are controlled simultaneously by using the $p$-value $p^j := \max_k \; p_{j, k}$.

\begin{proposition}
    Consider a single-step setting where we use risk calibration parameters $(\lambda, \theta) \in \hat \Phi$ to generate prediction sets and seek help whenever the prediction set is not a singleton. Let the upper bound on the help rate \eqref{eqn: human help rate} be set to $\alpha_2$. If the FWER-controlling parameter set $\hat \Phi$ is non-empty, then with probability $1-\delta$ over the sampling of the calibration set, the new scenarios drawn from $\mathcal D$ incur at most $\alpha_1$ rate of optimal action miscoverage and at most $\alpha_2$ rate of human help.   
\end{proposition}
\begin{proof}
    Follows directly from Proposition 6 of \cite{angelopoulos2021learn}.
\end{proof}
\subsection{Multi-Step, Single-Risk Control}

We now extend SRC to settings where a robot applies set-valued prediction in multiple time steps. This setting is useful for settings where the robot receives feedback from the human between steps. However, we cannot directly apply the above procedure because the help from the human changes the distribution $\mathcal D$ of state-intent pairs, and the i.i.d. assumption is no longer valid. We give an extension of the Learn-then-Test procedure to multi-step settings. 

\vspace{0.10cm}\noindent
\textbf{Sequence-Level Risk Calibration.} Similar to \cite{ren2023robots}, the key idea is to (i) lift the data to sequences and (ii) perform the LTT procedure using a carefully designed score function that allows for causal reconstruction of the prediction set at test time. We now consider a distribution $\bar{\mathcal D}$ of \textit{lifted} contexts induced by $\mathcal D$, where each lifted context contains a state-intent pair $(\Tilde{x}, \bar z) \sim {\bar{\mathcal D}}$. The lifted state-intent pairs are given as $\Tilde{x} = (\bar {x}_1, ..., \bar {x}_T)$ and $\bar z = (z_1, ..., z_T)$ respectively. Here, $\Tilde{x}$ arises from the robot having performed the \textit{correct} action in previous steps. Using the robot policy specified in Section IIIB, there are three cases to consider: the robot will (i)  take the  only available (and optimal) action if $\mathcal T_{\lambda, \theta} (\Tilde{{x}})$ is a singleton, or (ii) ask for clarification of the human's intent $z$ if the action set $\mathcal T_{\lambda, \theta} (\Tilde{{x}})$ is not a singleton. We bound the risk associated with case (iii): the optimal action is not in the prediction set as follows. Let $\bar u^{R*}(\bar z) := (u^{R*}(z_1), ..., u^{R*}(z_T))$ denote the \textit{sequence} of optimal robot actions. Let the sequence-level confidence be given as the lowest confidence over the timesteps
\begin{equation} \label{eqn: sequence level score function knowno}
    \bar g_{\theta}^*(\tilde{x}, \bar z) =  \; \; \min_{t \in [T]}  \; \; g_{\theta}^*(\bar{x}_t, z),
\end{equation}
where the corresponding \textit{sequence-level} prediction set is given as $\bar {\mathcal T}_{\lambda, \theta}(\tilde{x}) = \{\bar u \in \mathcal U^T: \exists \bar z \in \mathcal Z^T \; \text{ s.t. } \bar u = \bar u^{R*}(\bar z) \text{ and } \bar g_{\theta}^*(\tilde{x}, \bar z)  \geq \lambda) \}$. 


\vspace{0.10cm}\noindent
\textbf{Causal Reconstruction of Sequence-Level $\bar {\mathcal T}_{\lambda, \theta}$.} The sequence-level prediction set $\bar {\mathcal T}_{\lambda, \theta}$ is constructed with the full sequence $\bar z$ as labels, which depend causally on the sequence $\bar{x}$. Hence, we do not have the entire sequence $\bar z$ \textit{a-priori}; the robot must instead construct the prediction set at each time-step in a \emph{causal} manner (i.e., relying only on current and past observations). Let $\mathcal T_{\lambda, \theta}^t(\bar {x}_t) := \{u \in \mathcal U: \exists z \in \mathcal Z \; \text{ s.t. }  u =  u^{R*}(z) \text{ and }g_{\theta}^*(\bar{x}_t, z_t)  \geq \lambda) \}$ be the \textit{instantaneous} action prediction set at time $t$. We construct $\bar {\mathcal T}_{\lambda, \theta}$ in a causal manner using 
\begin{equation} \label{eqn: causal pred action set}
    {\mathcal T}_{\lambda, \theta}(\tilde{x}) := \mathcal T_{\lambda, \theta}^1(\bar {x}_1) \times \mathcal T_{\lambda, \theta}^2(\bar {x}_2) \times ... \times \mathcal T_{\lambda, \theta}^{T-1}(\bar {x}_{T-1}).
\end{equation} 


\begin{proposition}
    Consider a multi-step setting where we use risk calibration parameters $(\lambda, \theta) \in \hat \Phi$ and the sequence-level confidence \eqref{eqn: sequence level score function knowno} to generate sequence-level prediction sets and seek help whenever the prediction set is not a singleton. If the FWER-controlling parameter set $\hat \Phi$ is non-empty, then with probability $1-\delta$ over the sampling of the calibration set, the new scenarios drawn from $\bar{\mathcal D}$  under $\Pi^R$ and using the causally reconstructed predicted action set \eqref{eqn: causal pred action set} incur at most $\alpha_1$ rate of action miscoverage.  
    
\end{proposition}
\begin{proof}
    Let $(\lambda, \theta) \in \hat \Phi$, where $\hat \Phi$ controls the sequence-level FWER for the non-causal set $\bar {\mathcal T}_{\lambda, \theta}(\tilde{x})$ at level $\alpha_1$. We first show that $\bar u^{R*}(\bar z) \in \bar {\mathcal T}_{\lambda, \theta}(\tilde{x}) \Longleftrightarrow \bar u^{R*}(\bar z) \in {\mathcal T}_{\lambda, \theta}(\tilde{x})$. For any $\bar z \in \bar {\mathcal T}_{\lambda, \theta}(\tilde{x})$, 
\begin{equation} \label{eqn: sequence_causal_prediction_set_equality}
    \begin{aligned}
        \bar u^{R*}(\bar z) \in \bar {\mathcal T}_{\lambda, \theta}(\tilde{x}) & \Longleftrightarrow \min_{t \in [T]}  \; \; g_\theta^*(\bar{x}_t, z_t) \geq \lambda \\
        & \Longleftrightarrow g_\theta^*(\bar{x}_t, z_t) \geq \lambda \quad \forall t \\
        & \Longleftrightarrow u^{R*}(z_t) \in {\mathcal T}_{\lambda, \theta}^t(\bar {x}_t) \quad \forall t\\
        & \Longleftrightarrow \bar u^{R*}(\bar z) \in {\mathcal T}_{\lambda, \theta}(\tilde{x}).
    \end{aligned}
\end{equation}
Since the causally constructed prediction set is the same as the sequence-level prediction set, and since bound the risk associated with the sequence-level sets, we also bound the risk for the causally constructed sets. Applying the expectation definition of the risk \eqref{eqn: risk_definition_expectation} shows that the risks are the same. Applying fixed-sequence testing to the Hoeffding-Bentkus $p$-values completes the proof.
\end{proof} 
We now state our most general proposition for the multi-risk, multi-step setting. 

\subsection{Multi-Step, Multi-Risk Control}

In the multi-step, multi-risk setting, we seek to bound multiple risks simultaneously over the rollout of the robot policy $\Pi^R$ over $\bar{\mathcal D}$. However, the risk guarantee only holds for the lifted contexts in $\bar{\mathcal D}$ and are invalid if any distribution shift occurs from taking the wrong action.
In RCIP, distribution shift from $\bar{\mathcal D}$ (to some other distribution induced by taking suboptimal actions) may occur with the following probability,
\begin{equation} \label{eqn: OOD_risk}
    \mathbb P^{(\tilde x, \bar z) \sim \bar{\mathcal D}}(\text{OOD}) = \mathbb P^{(\tilde x, \bar z) \sim \bar{\mathcal D}}\Big(u^{R*}(z) \notin \mathcal T_{\lambda, \theta} (\bar{{x}}) \wedge|\mathcal T_{\lambda, \theta} (\bar{{x}})| \leq 1  \Big),
\end{equation}
i.e., when the optimal action is not covered by the prediction set and the prediction set is a singleton or empty, and thus the robot takes a non-optimal action. Here, we assume that the robot cannot take a suboptimal action if it asks for help. Eqn.~\eqref{eqn: OOD_risk} may be upper bounded by the action miscoverage rate $R_\text{cov}$ because it is the union of two events, but when $R_\text{cov}$ is large, distribution shift could be frequent.

In the multi-step, multi-risk setting, we consider a set of sequence-level risk signals $(R_1, ..., R_k)$  for contexts in $\bar{\mathcal D}$ bounded at nominal levels $(\alpha_1, ..., \alpha_k)$ by all $(\lambda, \theta) \in \hat \Phi$ as before. We assume that each risk models an event $E_k$, and the loss for each risk $L_k$ is an indicator function $\mathbbm{1}[E_k]$. We assume that $R_1 = R_\text{cov}$. In addition, since any OOD sequence incurs a task failure, we seek to bound the probability of $E_k$ occurring subject to an $R_1$ probability of distribution shift (in which case $E_k$ can also happen). 



\begin{proposition}
    Consider a multi-step setting where we use risk calibration with threshold level $(\lambda, \theta) \in \hat \Phi$ and the sequence-level score function \eqref{eqn: sequence level score function knowno} to generate sequence-level prediction sets and seek help whenever the prediction set is not a singleton. Consider a set of sequence-level risks $(R_1, ..., R_k)$ bounded at nominal levels $(\alpha_1, ..., \alpha_k)$, where $R_1$ is the miscoverage risk $R_\text{cov}$. If the action miscoverage rate is bounded at level $\alpha_1$ over the sampling of the calibration set, the new scenarios drawn from $\bar{\mathcal D}$ under $\Pi^R$ and using the causally reconstructed predicted action set \eqref{eqn: causal pred action set} incur at most $\alpha_1$ and $\alpha_k + \alpha_1$ rate of risk for $k \geq 2$ with failure rate $1-\delta$ over the sampling of the calibration set. 
\end{proposition}
\begin{proof}
    For $k=1$, risk $R_\text{cov}$ already provides a bound on the OOD rate. For $k\geq2$, the remainder of the proof follows a union bound argument. If the OOD rate is large, then the OOD-aware bound $\alpha_k$ will be much larger than the nominal bound. Therefore, the \text{OOD} rate must be controlled to have a non-trivial limit on the other risks. Using the definition of each risk and the linearity of expectation, we have that 
    \begin{equation} \label{eqn: union_bound}
    \begin{aligned}
        \alpha_1 + \alpha_k &\geq R_1(\lambda, \theta) + R_k(\lambda, \theta) \\
         &= \mathbb E^{(\tilde x, \bar z) \sim \bar{\mathcal D}} \Big[L_1(\tilde x, \mathcal T_{\lambda, \theta}(\tilde x), \bar z)  + L_k(\tilde x, \mathcal T_{\lambda, \theta}(\tilde x), \bar z) \Big ] \\
        &= \mathbb E^{(\tilde x, \bar z) \sim \bar{\mathcal D}} \Big[\mathbbm{1}[E_1] + \mathbbm{1} [E_k] \Big ] \\
        &= \mathbb E^{(\tilde x, \bar z) \sim \bar{\mathcal D}} \Big[\mathbbm{1}[E_1]\Big] + E^{(\tilde x, \bar z) \sim \mathcal D} \Big[\mathbbm{1} [E_k] \Big ] \\
        &= \mathbb P^{(\tilde x, \bar z) \sim \bar{\mathcal D}}(E_1) + \mathbb P^{(\tilde x, \bar z)\sim \bar{\mathcal D}}(E_k)\\
        &\geq \mathbb P^{(\tilde x, \bar z) \sim \bar{\mathcal D}} (E_1 \vee E_k).
    \end{aligned}
    \end{equation}
Then, either event $E_k$ or the event of distribution shift $E_1$ occurs at a rate no more than $\alpha_1 +\alpha_k$. 
\end{proof}

\begin{corollary}
    As a direct consequence of Eqn.~\eqref{eqn: union_bound}, if one wishes to calibrate risks other than the optimal action miscoverage rate, such as the user help rate \eqref{eqn: human help rate}, then it is sufficient to calibrate at level $\alpha_k = \max(\alpha_k^\prime - \alpha_1, 0)$, where $\alpha_k^\prime$ is the desired overall risk that incorporates distribution shift and the maximum is due to the constraint that the risk be in $[0, 1]$. 
\end{corollary}

%% file: subtex/experiments.tex
\section{Experiments} \label{Experiments}
\begin{figure}
    \centering
    \includegraphics[width=0.49\textwidth]{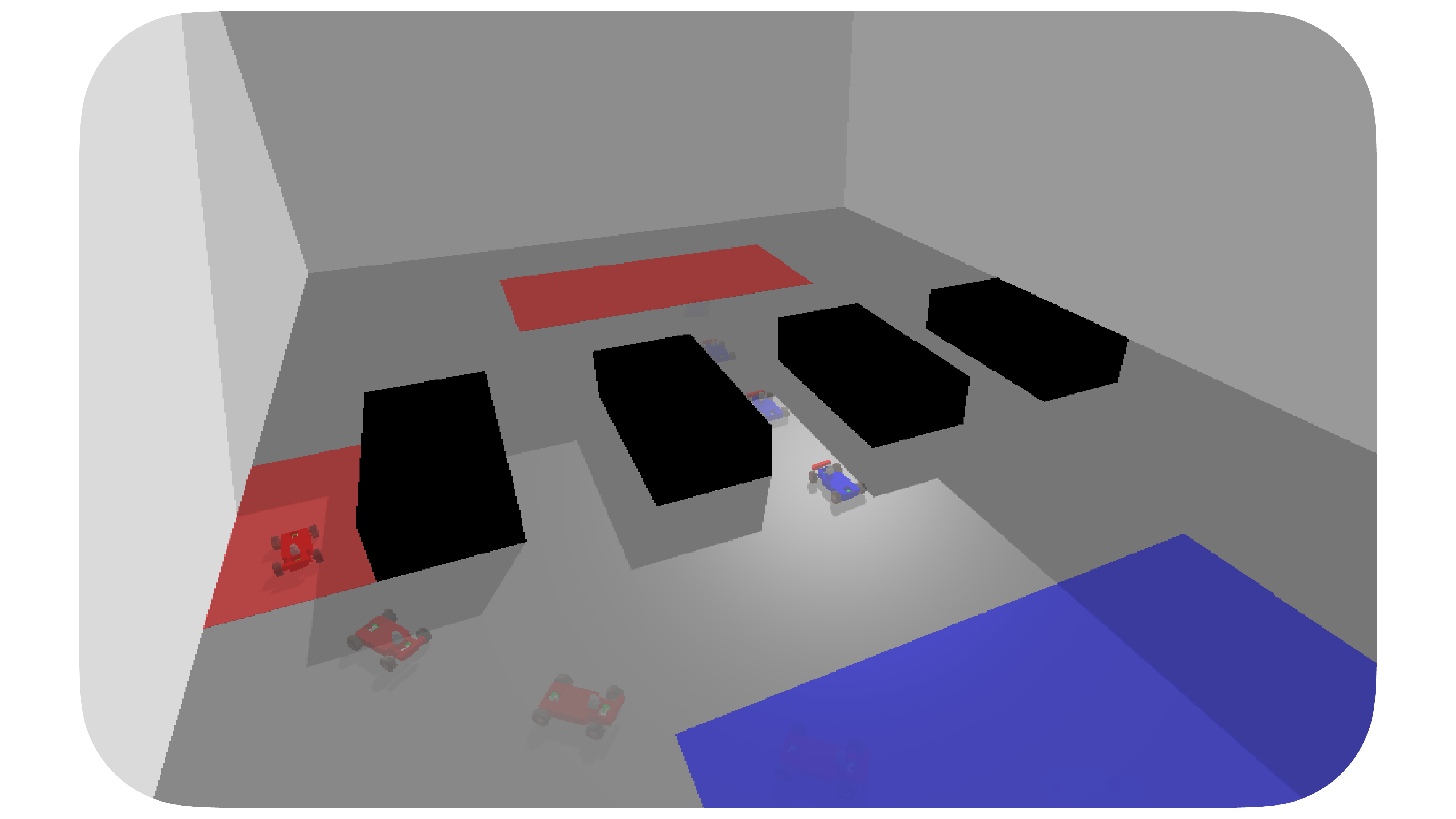}
    \caption{\rebuttal{Multi-step RCIP is applied in \textbf{Hallway Navigation}. The robot car (blue) and human car (red) are tasked with navigating to their respective goal states (large blue and red rectangles). The human car is constrained via its intent to pass through one of the five hallways (highlighted in red). The blue car does not observe the human's intent during evaluation.}}
    \label{fig: hallway}
\end{figure}

\vspace{0.10cm}\noindent 
\textbf{Environments.} \rebuttal{ We demonstrate single- and multi-step RCIP in four interactive domains, which exhibit four ways in which a robot planner can be integrated with an intent predictor.  First, we consider a multi-hallway setting in which two vehicles must coordinate to reach opposite ends of a room by navigating a set of hallways that are only one vehicle-width wide (see Fig.~\ref{fig: hallway}). One vehicle is driven by simulated human that has diverse intent. At the start of the episode, the human vehicle selects one of the hallways at random but does not communicate the hallway to the robot. Next, we investigate human-robot Social Navigation in close-quarters, cluttered household settings in the Habitat 3.0 \cite{puig2023habitat} simulator (see Fig.~\ref{fig: habitat_bimanual}(left)). Third, we show  hardware experiments for zero-shot Open-Category Sorting, in which the robot aids the human in sorting common household objects (e.g. books, toys, and fruit) by a mixture of size, shape, and color (see Fig.~\hyperlink{page.1}{1}). Finally, we show hardware experiments for a bimanual sorting setting in which the robot must take an object from the human and place it in the human's intended bin when multiple bins may be appropriate (see Fig.~\ref{fig: habitat_bimanual}(right)). In all settings, since the environment dynamics \eqref{eqn: dynamics} may evolve at a faster time scale than the human's intent dynamics \eqref{eqn: z_dynamics}, the human's intent is updated once every $T_z$ timesteps and is constant otherwise. }

\vspace{0.10cm}\noindent
\textbf{Scenario Distribution and Calibration Dataset.} RCIP can be used to obtain risk guarantees for an \textit{unknown} scenario distribution --- that is, of environments and human partners --- if can can collect i.i.d. samples from it for calibration. We envision that RCIP will enable a robot to interact with an end user (or set of users) through interactive data collection. Then, using the set of FWER-controlling parameters obtained from calibration (cf. section \ref{section: approach}), the user may set a level of autonomy for the robot depending on the risk limits of the task. The scenario distribution for each environment is described in the following subsections. Each calibration dataset is generated by random sampling from the environment distribution and from the distribution over human intents. For the simulation environments, we use a pre-trained prediction model using $10k$ random scenarios. For the hallway and Social Navigation environments, the prediction models are trained on a single NVIDIA GeForce RTX 2080 Ti GPU. Pre-training the prediction model takes about $4$ hours per environment. For calibration on hardware, data collection takes about $8$ hours. \rebuttal{For all environments, we fix $\delta=0.01$ and use a calibration dataset of size $M=400$.  We fix $\alpha_{\text{cov}} = 0.15$ for the Hallway and Social Navigation tasks and  We fix $\alpha_{\text{cov}} = 0.25$ for zero-shot prediction in the Open-Category Sorting and Bimanual Sorting tasks. In all experiments, we evaluate thresholds $\lambda \in [0, 1]$ using an evenly spaced grid with $|\hat \Lambda| = 2000$. We evaluate the model temperature $\theta \in [0.001, 10]$ using a log-spaced grid with $|\Theta| = 5$. For more details on the scenario distribution, please see Appendix \ref{Appendix: experiments}.}

\vspace{0.10cm}\noindent
\textbf{Baselines.} We compare RCIP against similar set-valued prediction approaches. A simple but naive approach for approximated $1-\alpha_{\text{cov}}$ coverage of optimal actions is \textbf{Simple Set}, which ranks actions according to a $1-\alpha_{\text{cov}}$ threshold using the predictor's raw confidence scores. Actions are sorted by greatest to least confidence, and actions are added to the prediction set in order of the sorted action set until the threshold is reached. To measure the effect of \textit{overall uncertainty} rather than individual scores, we compare against \textbf{Entropy Set}, which includes the highest overall prediction if the entropy of the distribution predicted actions is below a threshold; if not, then all actions are included in the prediction set, and the robot must ask for help. To evaluate the performance of vanilla conformal prediction against the richer hypothesis space of RCIP, we report results for \textbf{KnowNo} \cite{ren2023robots}. Similar in spirit but different from our work, KnowNo seeks to maximize coverage of optimal actions but without any guarantees on the human help rate, and assumes model parameters are fixed. Instead of maximizing coverage outright, RCIP balances prediction of optimal actions with limits on the human help rate, providing flexible performance guarantees depending on model parameters. Lastly, we consider \textbf{No Help} as an option, where the predicted action set always contains the predictor's most-confident action, and the human help rate is identically zero. 

\vspace{0.05cm}
\vspace{0.10cm}\noindent
 \textbf{Metrics.} For all environments, we report the task-level risks of (i) plan success rate and (ii) human help rate, on the test set. We also report the instantaneous risks --- measured as an average over time --- of plan success and human help. 
 

\subsection{Simulation: Hallway Navigation}

\begin{figure*}
    \centering
    \includegraphics[trim={0 5cm 0 5cm}, width=\textwidth]{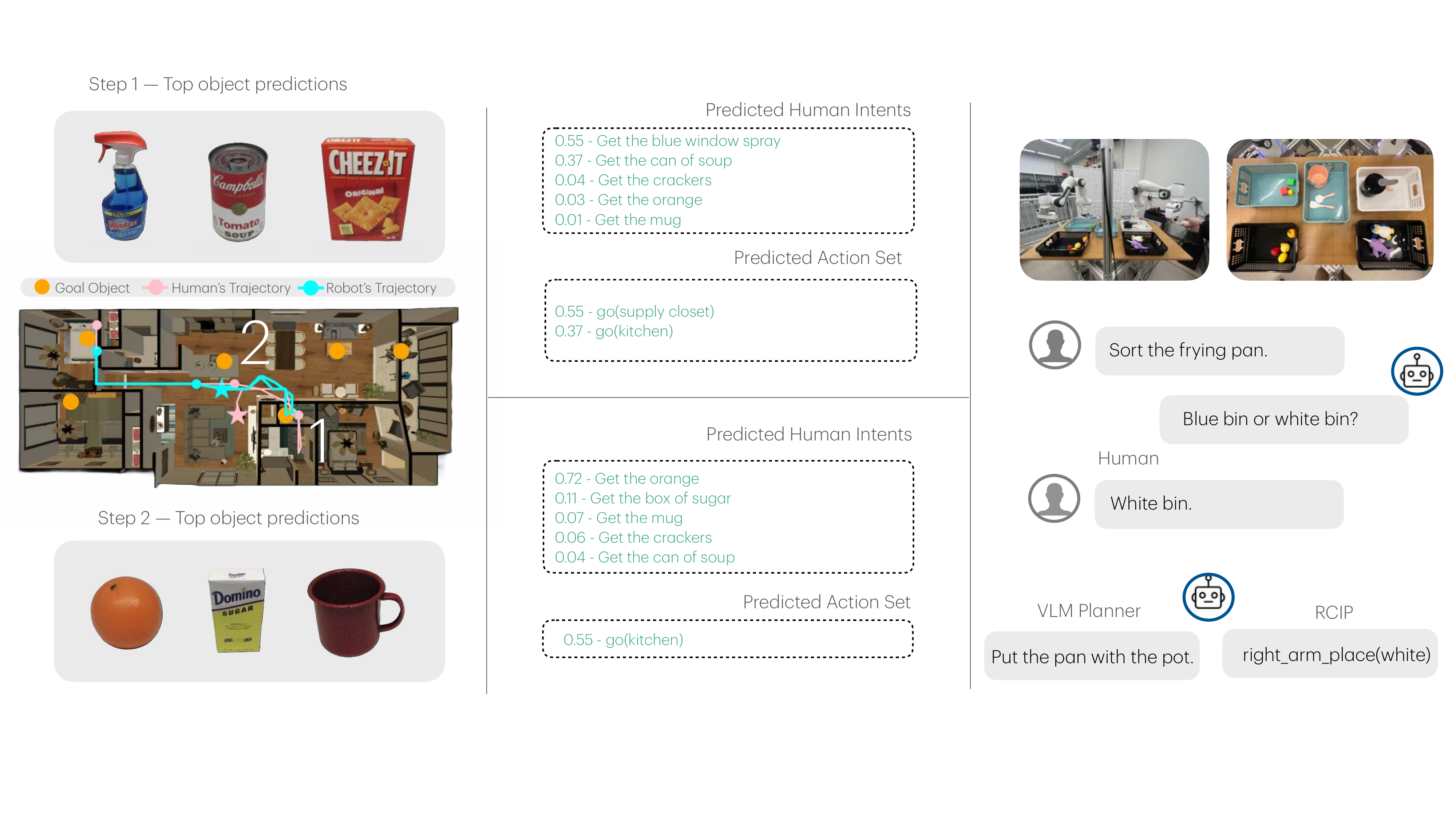}
    \caption{\rebuttal{(Left, Center) Multi-step RCIP is applied in \textbf{Social Navigation}. The human's trajectory is shown in pink, and the robot's trajectory is shown in blue. The human's possible goal objects are shown in orange. (Right) Single-step RCIP is applied in \textbf{Bimanual Sorting}. KnowNo, which generates plans in open-ended language, may generate a plan that is technically correct, but ambiguous to execute for a language-conditioned policy (both the blue and white bin have a pot). RCIP instead guarantees that the human's intent is satisfied via constraint satisfaction with the intent-conditioned planner.}}
    \label{fig: habitat_bimanual}
\end{figure*}
\begin{figure*}
    \centering
    \includegraphics[width=0.24\textwidth]{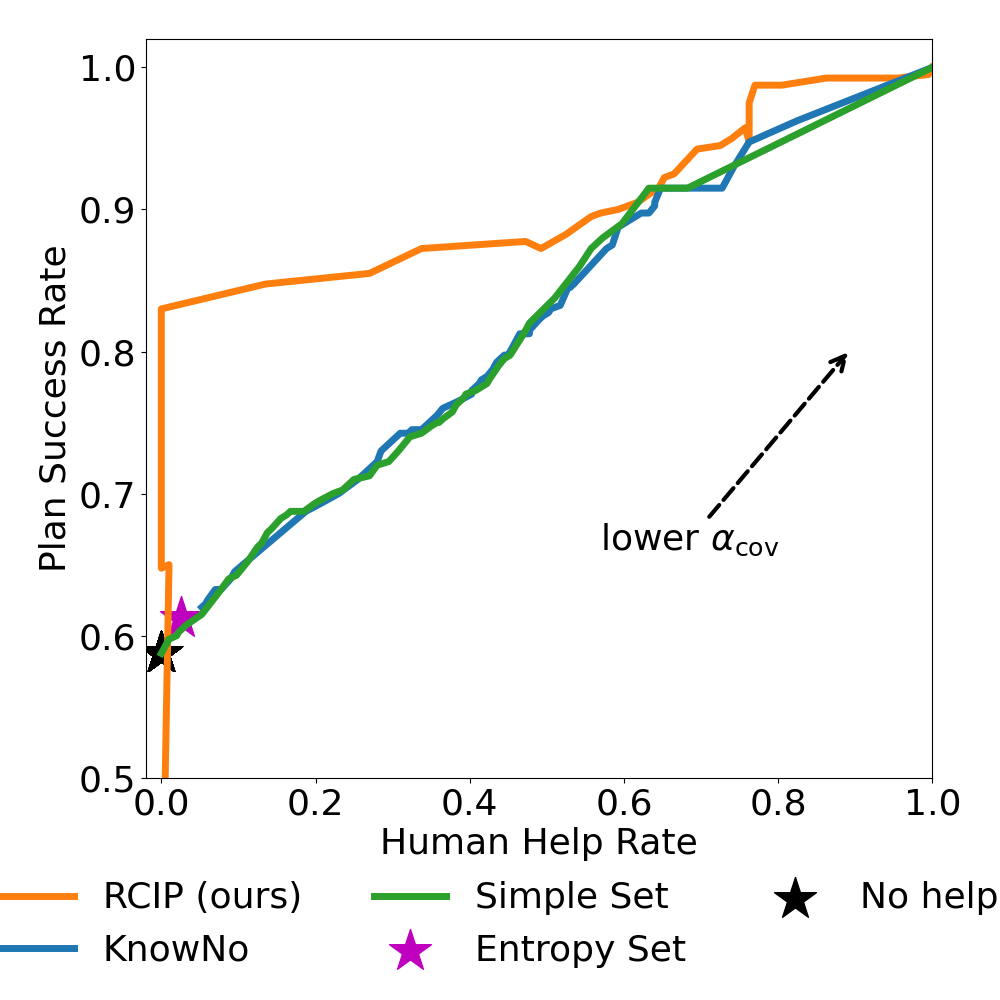}
    \includegraphics[width=0.24\textwidth]{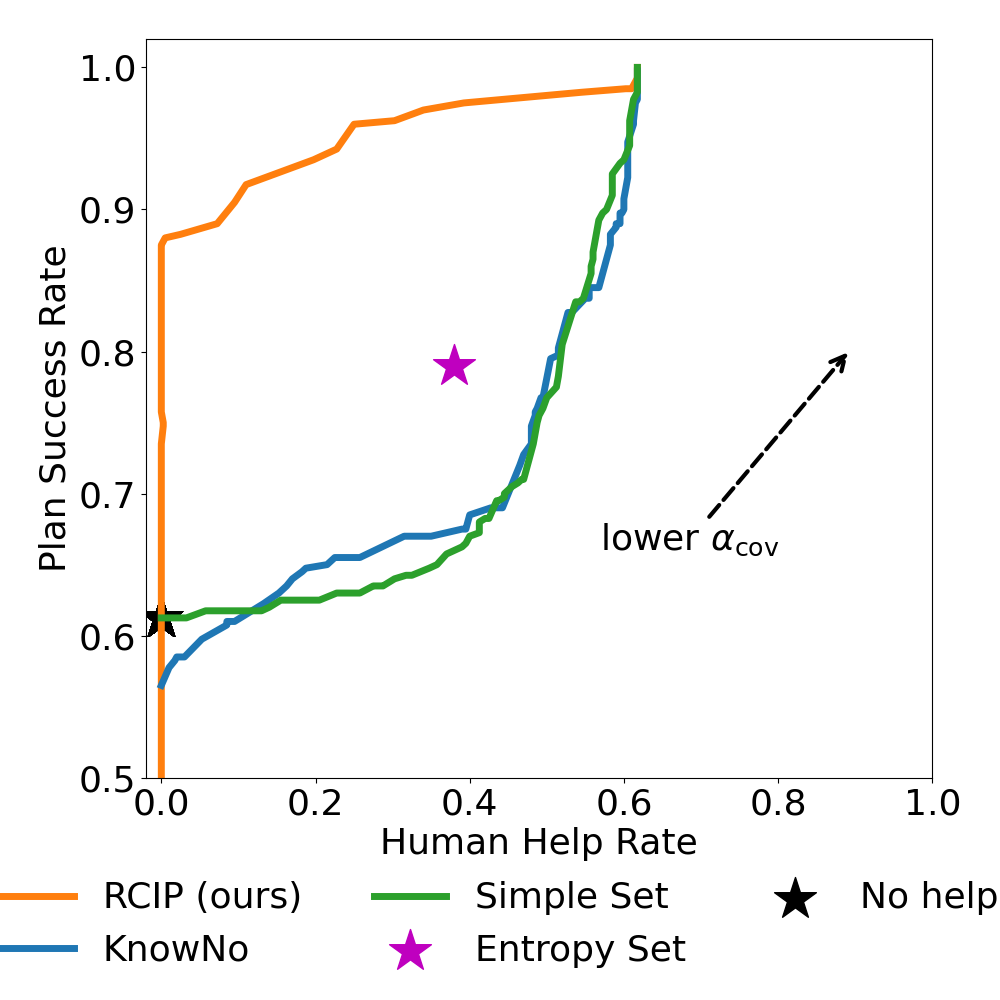}
    \includegraphics[width=0.24\textwidth]{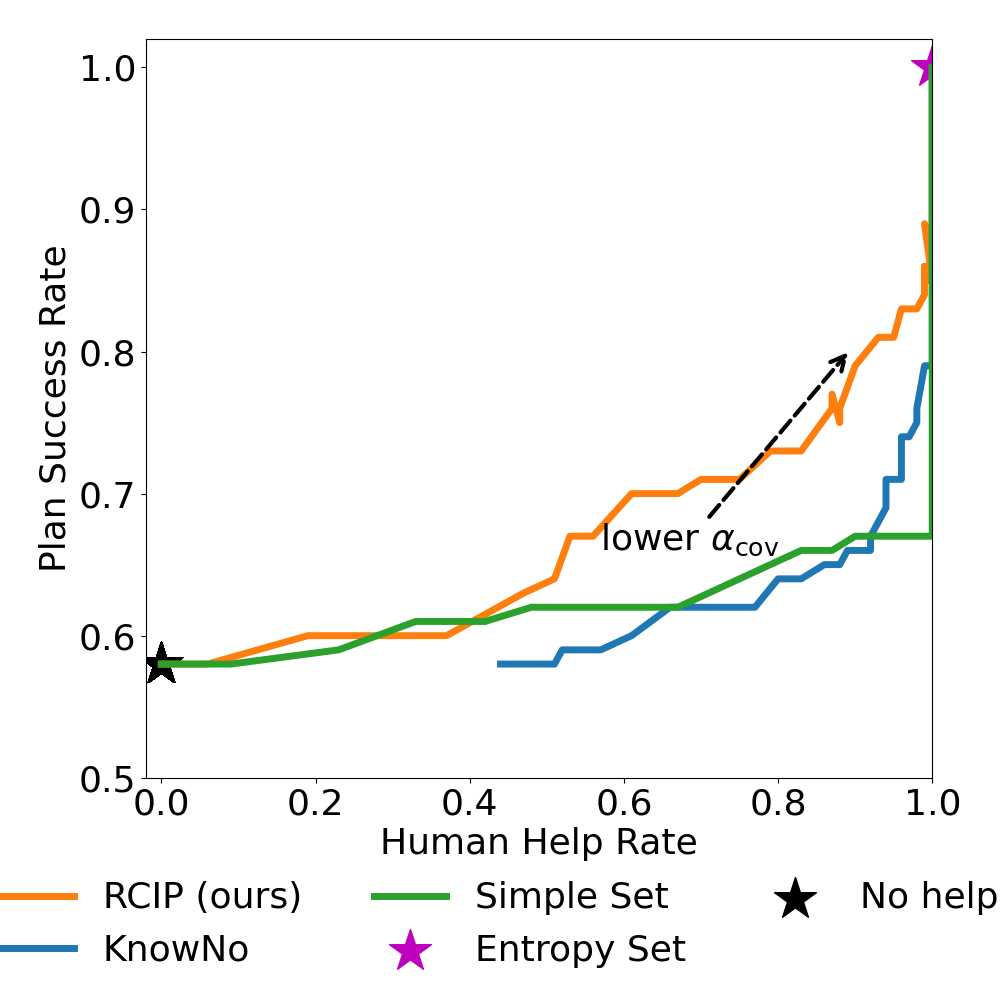}
    \includegraphics[width=0.24\textwidth]{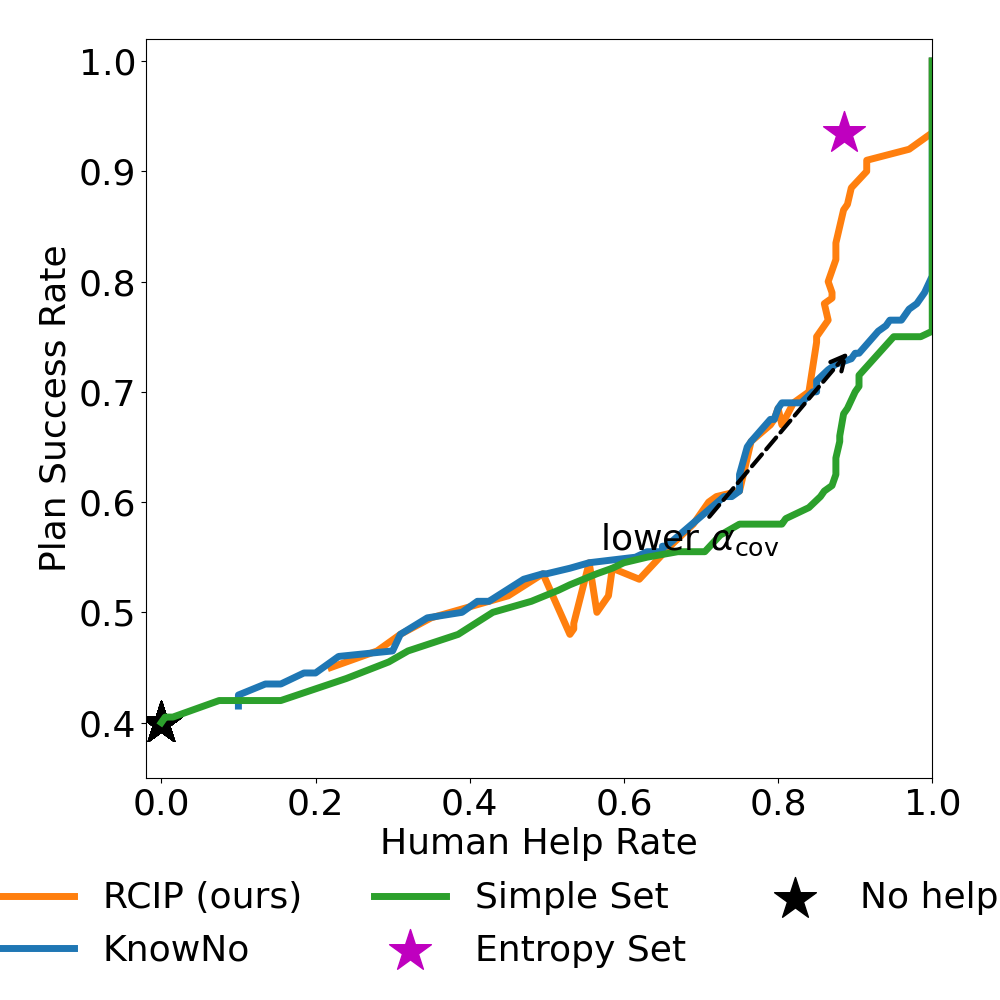}
    \caption{\rebuttal{Baseline comparison for RCIP versus other set-valued predictors for all tasks. RCIP consistently requires less help to achieve a specified plan success rate than other baseline methods. RCIP provides a framework for tuning model parameters to achieve risk control, versus other methods that assume that model parameters are held fixed: KnowNo \cite{ren2023robots}, Simple Set, Entropy Set, and No Help.}}
    \label{fig: baselines}
\end{figure*}

\setlength{\tabcolsep}{2pt}
\begin{table}[]
    \centering
    \begin{tabular}{cccccc}
        \hline 
        Method & $1-\alpha_{\text{cov}}$ & Plan Succ.$\uparrow$  & Plan Help$\downarrow$ & Step Succ.$\uparrow$ & Step Help$\downarrow$ \\
        \hline
        RCIP & $0.85$ & $0.86$ & $\textbf{0.27}$ & 0.98 & \textbf{0.27}  \\
        KnowNo \cite{ren2023robots} & $0.85$ & $0.87$ & $0.58$  & $0.98$ & $0.58$    \\
        Simple Set & 0.94 & $0.86$ & $0.54$ & $0.98$ & $0.54$  \\
        Entropy Set & $-$ & $0.61$ & $0.03$ & $0.94$ & $0.00$ \\ 
        No Help & $-$ & $0.59$ & $0$ & $0.94$ & $0$ \\ 
        \hline
    \end{tabular}
    \caption{Results for \textbf{Hallway Navigation}. The optimal action miscoverage rate is held fixed between RCIP, KnowNo, and Simple Set for comparing the other metrics.}
    \label{tab: hallway}
\end{table}

\setlength{\tabcolsep}{2pt}
\begin{table}[]
\rebuttal{
    \centering
    \begin{tabular}{c|cccccc}
        \hline 
        Model & Method & Plan Succ.$\uparrow$  & Plan Help$\downarrow$ & Step Succ.$\uparrow$ & Step Help$\downarrow$ \\
        \hline
        \multirow{2}{*}{1}& RCIP  & $0.86$ & $0.27$ & 0.98 & 0.27  \\
        & No Help  & $0.59$ & $0$ & $0.94$ & $0$ \\ 
        \hline
        \multirow{2}{*}{2}& RCIP  & $0.87$ & $0.04$ & 0.98 & \textbf{0.04}  \\
        & No Help  & $0.68$ & $0$ & $0.94$ & $0$ \\
        \hline
        \multirow{2}{*}{3}& RCIP  & $1.00$ & $0.00$ & 1.00 & 0.00  \\
        & No Help  & $1.00$ & $0$ & $1.00$ & $0$ \\
        \hline
    \end{tabular}
    \caption{\rebuttal{Help rate results for \textbf{Hallway Navigation} as the prediction model improves. RCIP is able to keep the plan success rate relatively constant even for poorly performing predictors. As the no-help predictions become better, RCIP also asks for less help.}}
    \label{tab: hallway_earlystopping}
}
\end{table}

Autonomous navigation around other autonomous decision-making agents, including humans, requires the robot to recognize scenario uncertainty (whether another agent will turn right or left) with task efficiency (energy spent braking or taking detours). While safety can almost always be guaranteed if each vehicle declares their intent at all times, such communication can be costly, especially if human prompting is involved. In this example (Fig.~\ref{fig: hallway}), the robot is asked to navigate to the initial condition of the human's vehicle without colliding. \rebuttal{The set of intents is $\mathcal Z = \{1, 2, 3, 4, 5\}$, where each intent corresponds to one of the five hallways. The confidence scores for each intent are computed by taking the temperature-weighted softmax scores for each hallway. The final action probabilities are computed according to Eqn.~\eqref{Eqn: predicted_action_set}. The robot interacts with the human over $T=200$ environment time steps and predicts the human's intent every $T_z=20$ time steps. Additional details on training the intent prediction model are deferred to Appendix \ref{Appendix: Models}.}  

To ensure that the robot reaches its goal state in a minimal amount of time, we permit the robot to prompt the human for their chosen hallway if its optimal action set is not a singleton. We jointly learn the robot and human policies using proximal policy optimization (PPO) \cite{schulman2017proximal, yu2022surprising}. The human and robot PPO policies are trained jointly using $256$ environments and take about $4$ hours to train. 

\begin{figure}
    \centering
    \includegraphics[width=0.35\textwidth]{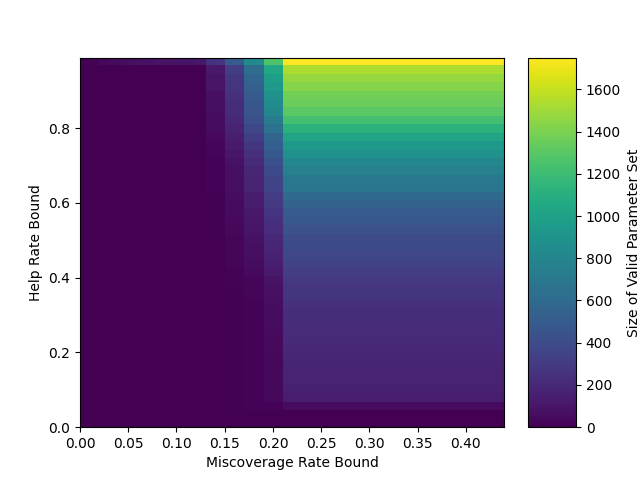}
    \caption{Ablation study on the effect of action miscoverage and help rate risk limits versus FWER-controlling parameter set size for RCIP on \textbf{Hallway Navigation} using $\alpha_\text{cov} \in [0, 0.45]$ and $\alpha_\text{help} \in [0, 1]$. The color denotes the size of the set of FWER-controlling parameters $\hat \Phi$, with empty (infeasible) sets taking a size of zero.}
    \label{fig: ablation}
\end{figure}


We present results for Hallway Navigation in Table~\ref{tab: hallway} at the fixed coverage level $\alpha_{\text{cov}}=0.85$. We additionally study various model calibration levels and their effect on the RCIP help rate. We present results in  Table~\ref{tab: hallway_earlystopping}

Fig.~\ref{fig: baselines} provides a comparison between RCIP and other baseline approaches that employ set-valued prediction. While entropy and simple-set can be used to provide (respectively) static and dynamic thresholds for heuristic uncertainty quantification, these uncalibrated methods often ask for too much help and scale poorly as the desired plan success rate increases.

Fig.~\ref{fig: ablation} provides an ablation study on the effect of the bounds on miscoverage and the human help rate on the size of the FWER-controlling parameter set $|\hat \Phi|$ in the \textit{multi-risk, multi-step} setting. As the miscoverage rate bound becomes lower, lowering the human help rate provides fewer valid parameters, until $|\hat \Phi| = 0$, and controlling both risks is infeasible.

\setlength{\tabcolsep}{2pt}
\begin{table}[]
    \centering
    \begin{tabular}{cccccc}
        \hline 
        Method & $1-\alpha_{\text{cov}}$ & Plan Succ.$\uparrow$  & Plan Help$\downarrow$ & Step Succ.$\uparrow$ & Step Help$\downarrow$ \\
        \hline
        RCIP & $0.85$ & $0.86$ & $\textbf{0.00}$ & 0.90 & \textbf{0.00}  \\
        KnowNo \cite{ren2023robots} & $0.85$ & $0.88$ & $0.58$  & $0.95$ & $0.58$    \\
        Simple Set & 0.83 & $0.86$ & $0.56$ & $0.94$ & $0.56$  \\
        Entropy Set & $-$ & $0.79$ & $0.38$ & $0.93$ & $0.21$ \\ 
        No Help & $-$ & $0.61$ & $0$ & $0.81$ & $0$ \\ 
        \hline
    \end{tabular}
    \caption{Results for \textbf{Social Navigation}. The optimal action miscoverage rate is held fixed between RCIP, KnowNo, and Simple Set for comparing the other metrics.}
    \label{tab: habitat}
\end{table}


\subsection{Simulation: Social Navigation in Habitat}

Habitat \cite{puig2023habitat} is a photo-realistic simulator containing a diverse set of scenes, objects, and humans models for human robotics tasks. In this experiment, a Boston Dynamics Spot robot and human are jointly tasked with navigating to a set of goal objects in sequence, to simulate cleaning up a house (i.e., grabbing various items, such as crackers, cans of soup, etc. as shown in Fig.~\ref{fig: habitat_bimanual}). Each scene contains $5-10$ objects of interest \rebuttal{from the YCB dataset} \cite{calli2015benchmarking}. Although the human may initially be out of view of the robot, the robot must find the human and maintain a safe distance of one meter at all times. \rebuttal{We simulate the human's decision making by choosing a high-level intent from the set of objects; here, the set of intents is $\mathcal Z=[N_o]$, where $N_o$ is the number of objects in the scene. The confidence scores for each intended object are computed by taking the temperature-weighted softmax scores for each goal object. The final action probabilities are computed according to Eqn.~\eqref{Eqn: predicted_action_set}. The robot interacts with the human over $T=600$ environment time steps and selects a new goal object every $T_z=100$ time steps. } 

Since the human's goal object is not observed by the robot, one naive strategy is to navigate to the human first, then follow the human around the house. However, since the scene is cluttered, remaining too close to the human could impede their progress (e.g. getting in the way) or block the robot, resulting in suboptimal, unsafe behavior.  By predicting the human's motion, the robot is able to better accommodate the human's task while remaining safe (with high probability) with respect to unsafe interactions. We present results for Social Navigation in Table~\ref{tab: habitat}. \rebuttal{We defer additional discussion, such as the effect of temperature scaling, to Appendix \ref{Appendix: discussion}}.

\subsection{Hardware: Open-Category Sorting}

In this example (Fig.~\hyperlink{page.1}{1}), each scenario tasks the robot with helping a human to sort a set of objects by inferring the sorting category for each object. 
\rebuttal{Since the human may have a preference for how the robot sorts the objects, the robot's trajectory is constrained such that the object must be placed in the correct bin. The human's intent space $\mathcal Z=\{1, 2, 3\}$, where each number corresponds to one of the bins. The robot interacts with the human over $T=6$ environment time steps, and the human selects a new sorting plan every $T_z=1$ time step.}


\rebuttal{To predict the human's intent, the robot takes in an image observation of the table and has access to a (vision) language model to process the semantic features of the image. We use GPT-4V (\texttt{gpt-4-turbo-2024-04-09}) \cite{achiam2023gpt} to process the image and predict (e.g. the carrot in Fig.~\hyperlink{page.1}{1}) a confidence score for each bin via multiple-choice question and answering (MCQA) \cite{srivastava2022beyond, hendrycks2020measuring}. The temperature-weighted softmax scores for each bin give the final action probabilities. }

For safety, we restrict the robot and human to work in separate workspaces, such that the human only places objects inside the human workspace and the robot only places block in the robot workspace (i.e., there is no shared workspace).  To warm-start the predictor, we allow the human to initially place $3-10$ objects, with eight more to place, for a total of up to $30$ objects per task. We show in Table~\ref{tab: sorting} that RCIP reduces the plan-wise help rate by $11\%$ and the step-wise help rate by $5\%$ in Open-Category Sorting. We use a Franka Emika Panda arm for the robotic manipulation  portion of the task. Images of the scene (for both perception and planning) are obtained from an Azure Kinect RGB-D camera. \rebuttal{ For more hardware implementation details, please see Appendix \ref{Appendix: implementation}.}

\setlength{\tabcolsep}{2pt}
\begin{table}[]
    \centering
    \begin{tabular}{cccccc}
        \hline 
        Method & $1-\alpha_{\text{cov}}$ & Plan Succ.$\uparrow$  & Plan Help$\downarrow$ & Step Succ.$\uparrow$ & Step Help$\downarrow$ \\
        \hline
        RCIP & $0.75$ & $0.76$ & $\textbf{0.87}$ & 0.94 & \textbf{0.93}  \\
        KnowNo \cite{ren2023robots} & $0.75$ & $0.75$ & $0.98$  & $0.95$ & $0.98$    \\
        Simple Set & 0.59 & $0.76$ & $1.00$ & $0.96$ & $1.00$  \\
        Entropy Set & $-$ & $1.00$ & $1.00$ & $1.00$ & $1.00$ \\ 
        No Help & $-$ & $0.58$ & $0$ & $0.88$ & $0$ \\ 
        \hline
    \end{tabular}
    \caption{Results for \textbf{Open-Category Sorting}. The optimal action miscoverage rate is held fixed between RCIP, KnowNo, and Simple Set for comparing the other metrics.}
    \label{tab: sorting}
\end{table}

\subsection{\rebuttal{Hardware: Bimanual Sorting}}

\rebuttal{In this experiment (Fig. \ref{fig: habitat_bimanual}),  a bimanual setup with two Franka Panda arms is asked to help the human sort a common household item: frying pan, plushie, wooden block, carrot, or grape. We initialize each bin to contain similar objects; for example, bins can contain other cookware, plushies, wooden blocks, vegetables, or fruit.  There are five bins total: two bins are on each side of the table, and one bin is in the middle of the table. The workspace of both arms is limited such that each arm cannot reach the two bins at the other end of the table;  thus, there is ambiguity arising from which arm the human wants to help sort the item. Additionally, duplicate categories provide intent ambiguity: each of the initializations is selected at random and may appear up to twice in the five bins. The robot must then interpret a human instruction and infer the set of bins that it can place the object into (Bin 1, Bin 1 or Bin 2, Bin 1, Bin 3, or Bin 5, etc). Therefore, the human's intent space $\mathcal Z$ is the \textit{power set} of possible bins, i.e., $\mathcal Z = \mathcal P(\{1, 2, 3, 4, 5\})$. Hence, there are $|\mathcal Z|=31$ possible intents for this task (when excluding the empty set). The robot interacts over a single time step, so $T = T_z = 1$.  We show in Table~\ref{tab: bimanual} that RCIP reduces the plan-wise help rate by $11\%$ and the step-wise help rate by $5\%$.  }

\setlength{\tabcolsep}{2pt}
\begin{table}[]
\rebuttal{
    \centering
    \begin{tabular}{ccccc}
        \hline 
        Method & $1-\alpha_{\text{cov}}$ & Plan Success $\uparrow$ & Execution Success $\uparrow$ & Plan Help $\downarrow$ \\
        \hline
        RCIP & 0.75 & 0.75 & 0.67 & \textbf{0.85} \\ 
        KnowNo \cite{ren2023robots} & 0.75 & 0.75 & 0.60 & 0.93 \\
        Simple Set & 0.75 & 0.75 & 0.67 & 0.95 \\
        Entropy Set & $-$ & 0.94 & 0.83 & 0.88 \\ 
        No Help & $-$ & 0.40 & 0.27 & 0\\ 
        \hline
    \end{tabular}
    \caption{\rebuttal{Results for \textbf{Bimanual Sorting}. The optimal action miscoverage rate is held fixed between RCIP, KnowNo, and Simple Set for comparing the other metrics. We use a random set of 30 scenarios from the test set to evaluate the execution success rate for each method.} }
    \label{tab: bimanual}
}
\end{table}


%% file: subtex/limitations.tex
\section{Limitations and Future Work}

The primary limitation of our work is a lack of guarantee on the low-level execution of the controller. Concretely, if the correct optimal action is predicted by the robot, but the controller fails to execute the computed command, then the robot will execute a suboptimal action and encounter a distribution shift, invalidating the results from RCIP. In the future, we are looking to incorporate low-level control failures as part of the risk calibration procedure. Additionally, our work fundamentally assumes that the human's intent is drawn from a finite set, and moreover that the human's intent is verbalizable or clarifiable (i.e. the human is able to provide meaningful clarifications when the robot asks for help). 

In the future, we hope that RCIP can be combined with active preference learning \cite{sadigh2017active, eric2007active, wilde2020active} to better incorporate the human's preferences in determining the appropriate level of robot autonomy (e.g. choosing from the valid set of RCIP parameters). We also plan to study RCIP's ability to capture higher levels of interactivity in a system, such as when the robot must operate around more than one human, or when some humans are non-cooperative.

%% file: subtex/conclusion.tex
\section{Conclusion} 

We propose Risk-Calibrated Interactive Planning (RCIP), a framework that applies statistical multi-hypothesis risk control to address the problem of risk calibration for interactive robot tasks. We formalize RCIP as providing a statistical guarantee on an arbitrary number of user-specified risks, such as prediction failures and the amount of human help, subject to a bound on the rate at which the robot fails to predict the optimal actions. By optimizing preferences over a small number of model parameters, RCIP is able to achieve higher flexibility in aligning to user preferences than fixed-paramter methods. Experiments across a variety of simulated and hardware setups demonstrate that RCIP does not exceed user-specified risk levels. \rebuttal{ Moreover, RCIP reduces user help $8-87\%$ when compared to baseline approaches that lack formal assurances. }

%% file: subtex/appendix.tex
\subsection{Additional Experiment Details} \label{Appendix: experiments}

\begin{figure*}
    \centering
    \includegraphics[width=0.6\textwidth]{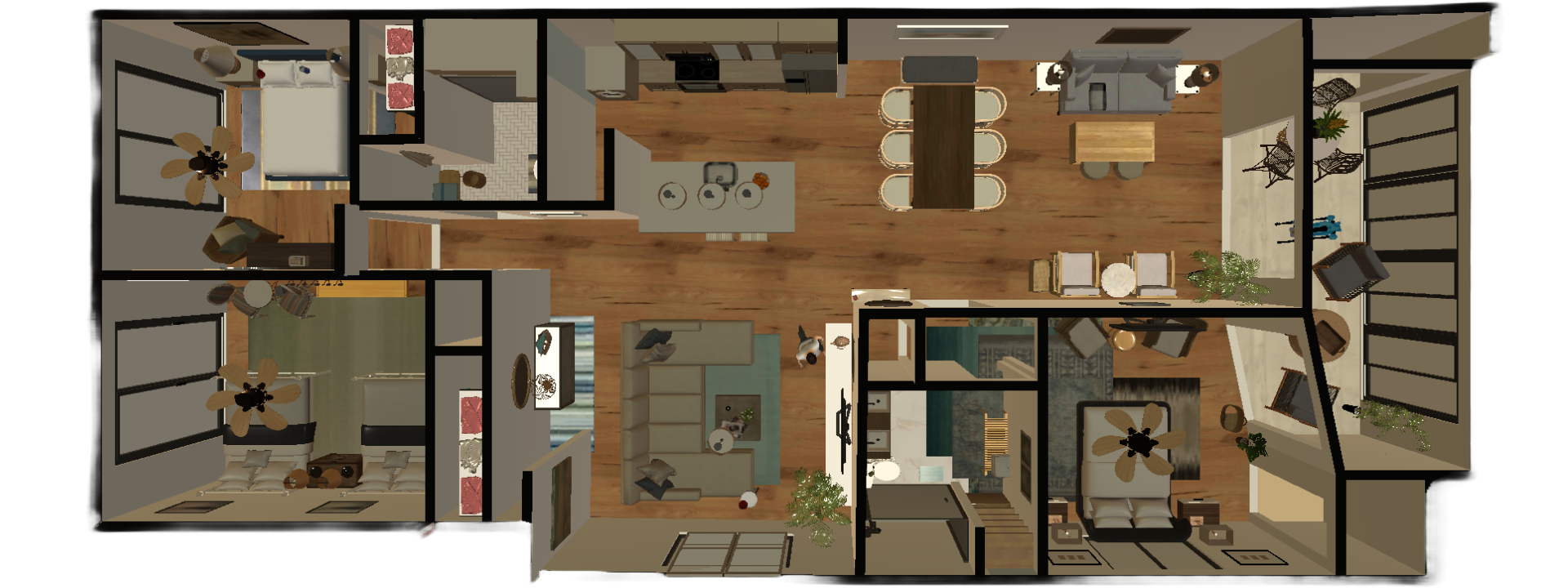}
    \includegraphics[width=0.6\textwidth]{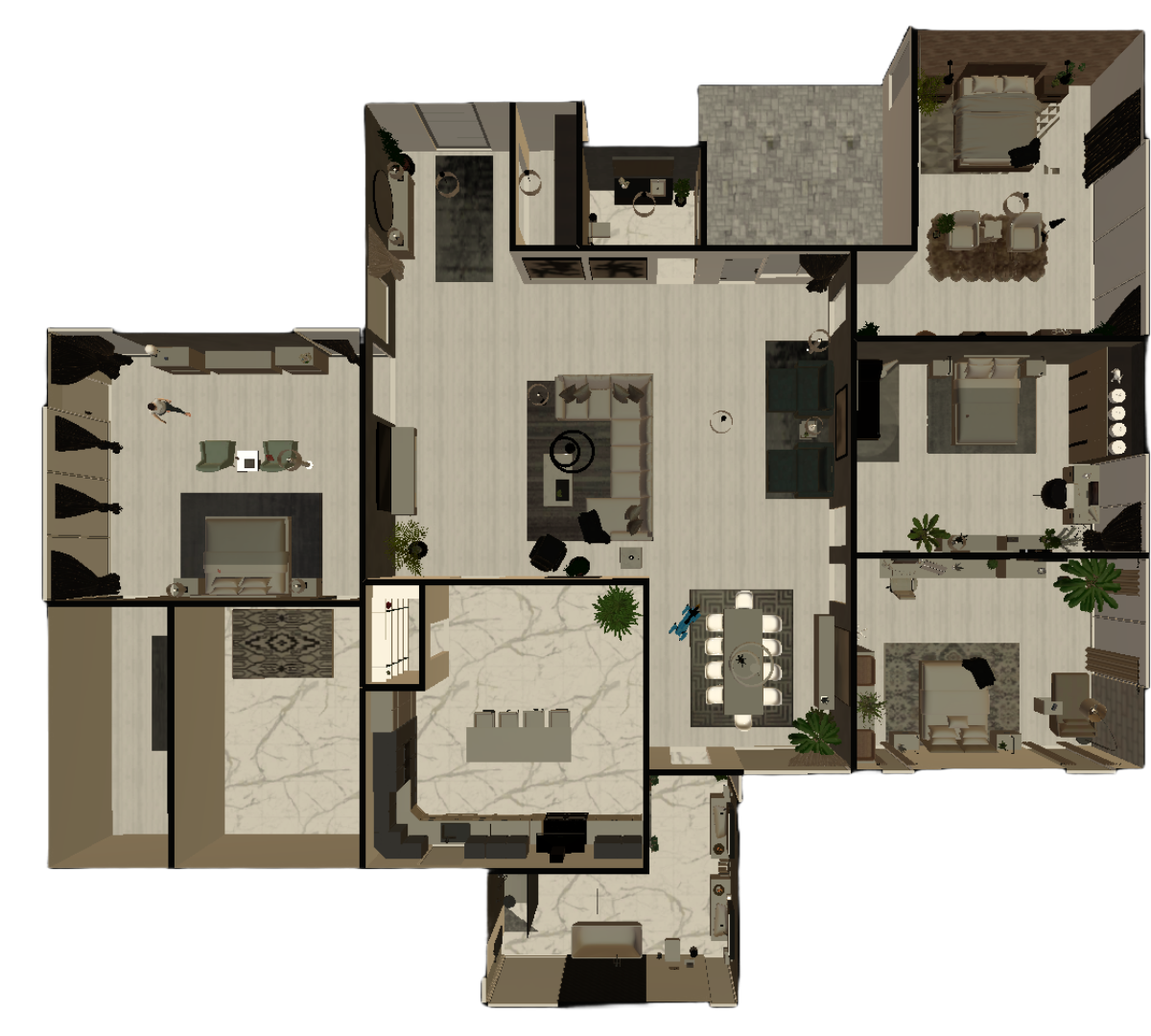}
    \includegraphics[ width=0.6\textwidth]{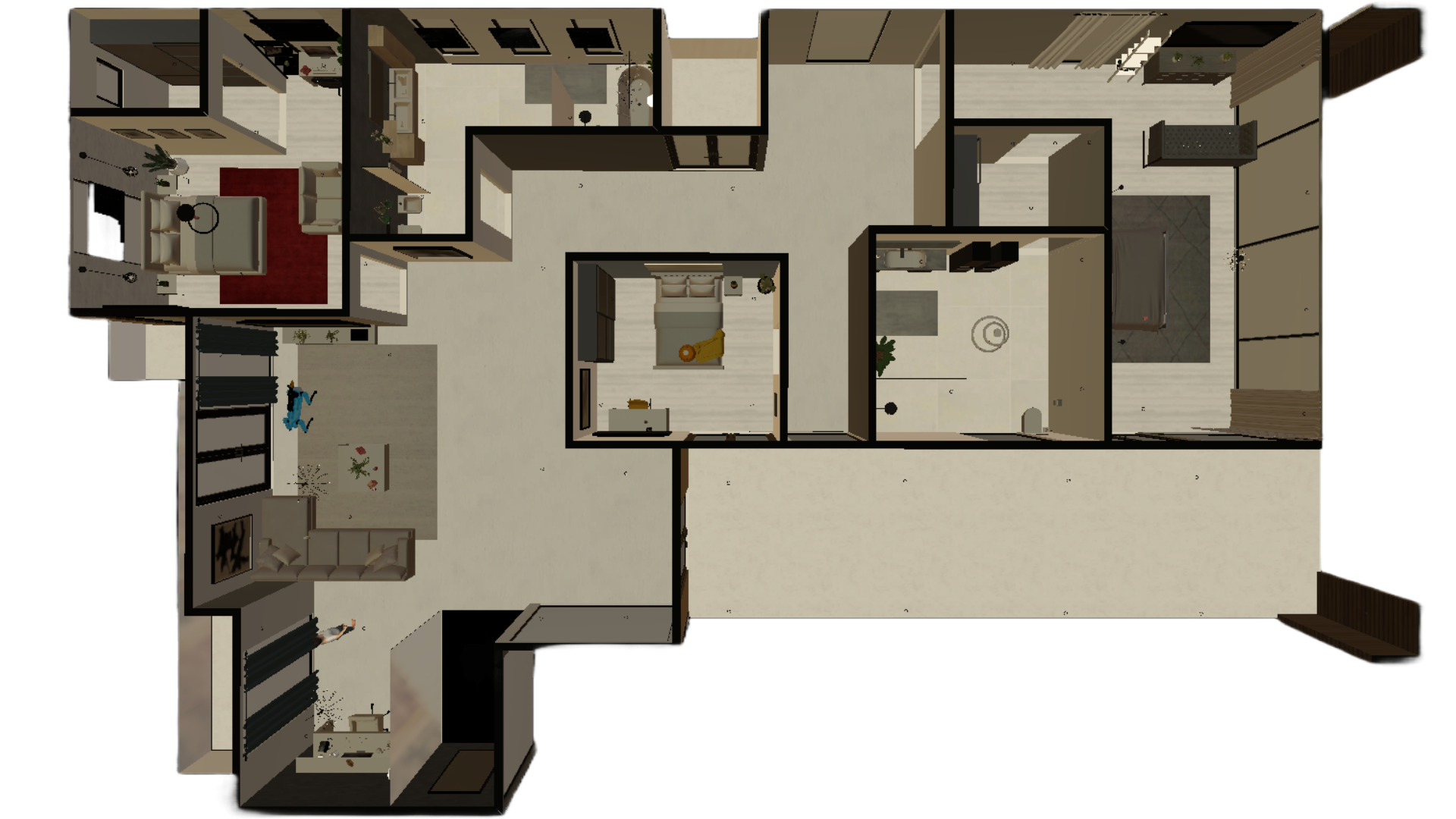}
    \caption{Three sample maps from \textbf{Social Navigation}. The human and robot are tasked with finding and collecting a series of objects (e.g. cans of soup, crackers) around a close-quarters living space. The robot must recognize the human's intent and either follow or evade the human depending on the human's desired object. The robot must minimize action miscoverage across a variety of human intents and environments.}
    \label{fig: habitat}
\end{figure*}

 \begin{figure*}[t!]
     \centering
     \includegraphics[width=0.6\textwidth]{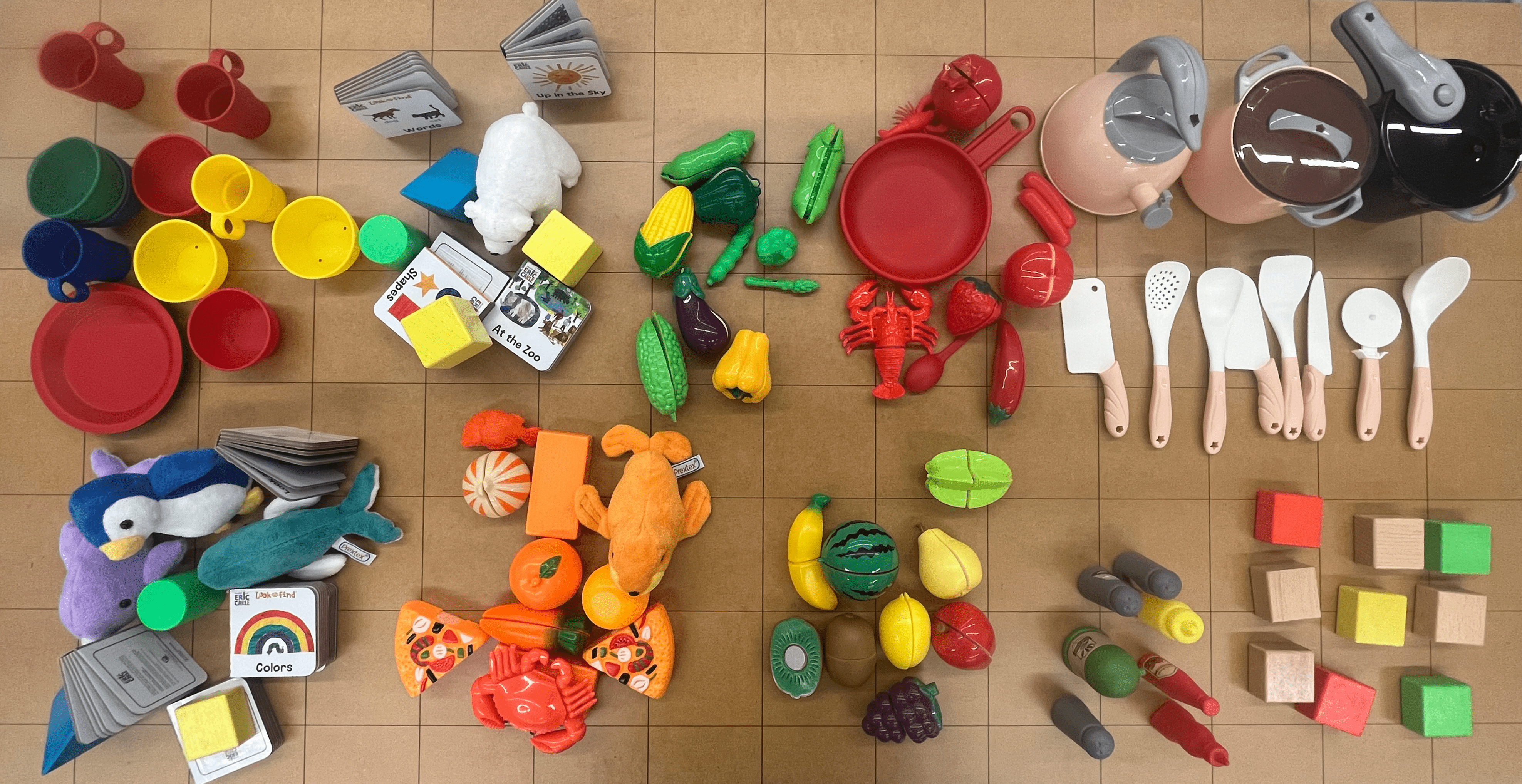}
     \caption{Full object set for Open-Category Sorting and Bimanual Sorting}
     \label{fig:full_object_set}
 \end{figure*}


\noindent
\textbf{Environments.} In addition to Fig.~\hyperlink{page.1}{1}, Fig.~\ref{fig: hallway}, Fig.~\ref{fig: habitat_bimanual} (left) and Fig.~\ref{fig: habitat_bimanual}(bimanual). Fig.~\ref{fig: habitat} shows the map and object set from the Social Navigation task. Fig.~\ref{fig:full_object_set} shows the full object set from the hardware Open-Category Sorting and Bimanual Sorting.

\noindent
\textbf{Scenario Distribution and Calibration Dataset.} Here, we provide more details on the on the parameterization of the scenario distribution for the four experiment settings. For each setting, we state the time horizon (single or multi-step), the environment details, the human intent distribution, and the robot's goal. With each distribution defined, the calibration datasets are then generated by sampling 400 i.i.d. scenarios from the distributions over scenarios. Specific to the Open-Category Sorting and Bimanual Sorting tasks, we augment the RCIP context with a language instruction since these tasks require querying a vision-language model. Each setting is as follows:

\noindent
\textbf{Simulation: Hallway Navigation.} 
\begin{itemize}
    \item[--] Horizon: multi-step (10 steps – 1 step per 20 environment timesteps).
    \item[--]Environment: two cars placed in opposite ends of an area that is 16 meters long and  9 meters wide. The arena is bottlenecked by five 1-meter wide parallel hallways spaced evenly throughout the arena. Each car is 0.5 meters wide and has a safety radius of 0.5 meters, and therefore two cars cannot take the same hallway without violating safety specifications. Each car is initialized with a random position and heading within their respective sides of the arena. 
    \item[--] Intents: Each of the 5 hallways, with equal spacing between each hallway. The intent set Z is \{1, 2, 3, 4, 5\}. At each episode, the intent is sampled from Z with uniform probability.
    \item[--] Task: reach the other end of the arena without colliding with the human or with the walls. 
\end{itemize}

\noindent
\textbf{Simulation: Social Navigation.} 
\begin{itemize}
    \item[--] Horizon: multi-step (6 steps – 1 step per 100 environment timesteps).
    \item[--] Environment: Our Habitat simulation environment consists of 77 objects from the YCB dataset \cite{calli2015ycb}, up to 10 of which are randomly sampled per scene. Scenes are randomly sampled from the  Habitat Synthetic Scenes Dataset  \cite{khanna2023habitat}, which is a dataset of 211 high-quality 3D scenes containing over 18k individual models of real-world objects. The human and robot are placed in one of three randomly selected, feature-rich environments shown in Figure  Each environment contains a variety of objects that the human can interact with, shown in Figure . The robot has a safety radius of 1 meter such the task is failed if the safety radius is violated through contact with the human. There are several hallways, doors, and passageways that the robot and human cannot simultaneously pass through without colliding. 
    \item[--] Intents: Each of the $N_o$ objects in the scene. The intent set $\mathcal Z=\{1, 2, 3, \dotsc, N_o\}$. At each episode, the initial intent is sampled from $\mathcal Z$ with uniform probability. Then, every 100 timesteps, the human chooses a new object. 
    \item[--] Task: reach a series of human goal objects without colliding with the human. 

\end{itemize}

\noindent
\textbf{Hardware: Open-Category Sorting} 
\begin{itemize}
    \item[--] Horizon: multi-step (6 steps).
    \item[--] Environment: there are eight items to be sorted placed randomly on the table, and there are three bins: white, black, and blue. The human initially fills each bin with 5-10 objects from the object set shown in Fig.~\hyperlink{page.1}{1}. Each bin satisfies a single grouping from the set of 10 groupings shown in Fig.~\hyperlink{page.1}{1}. The groupings are: fruit, vegetables, toys (set 1), toys (set 2), red, orange, cookware, tableware, sauces, and blocks.
    \item[--] Intents: The human’s intent is represented as a choice of one of the 3 bins. The intent set Z is \{1, 2, 3\}. There is considerable ambiguity between the categories for many scenarios; for example, a carrot is both orange and a vegetable (see Fig. 1). GPT-4V is asked to score the human’s intent based on an image of the object to be sorted and three images of the goal bins.
    \item[--] Task: place the eight items in the bin that the human thinks is the best match for the object. 
    \item[--] Language context: To augment the RCIP context, GPT-4V is given the following prompt, “The first three images are different bins sorted by a human. Describe the items in each of the bins." We additionally provide the prompt, “Here is a description of three bins and an object we want to sort. {VLM description} Which bin should we place the object in?” The 3 choices are then listed via MCQA. 
\end{itemize}

\noindent
\textbf{Hardware: Bimanual Sorting} 
\begin{itemize}
    \item[--] Horizon: single-step.
    \item[--] Environment: there are five objects that the human can ask the robot to sort: frying pan, carrot, grape, block, and blue plushie. One of the objects is selected uniformly at random. There are five bowls placed on the table. Each bowl contains objects from the following categories: cooking, fruits, vegetables, blocks, and plushies. Three bowls are fixed (fruits, cooking, and plushies), and two bowls are drawn uniformly at random (possibly with duplication) from the set of categories. There is a human standing at the left, middle, or right side of the table, with their position being drawn uniformly at random.
    \item[--] Intents: tower set of the five bowls for the object to be placed, $\mathcal P(\{1, 2, 3, 4, 5\})$. For example, if there are two vegetable bowls {1, 2} on the left and right side, the human could instruct the robot “place the carrot with the vegetables”, and could mean either vegetable bowl {1, 2}. Thus, there are 31 intents total that each represent a combination of possible placements (the empty set is excluded). 
    \item[--] Task: move the object to the specified location that satisfies the human intent. Minimize the distance between the object and the closest intended bin. 
    \item[--] Language instruction (part of RCIP context):  To augment the RCIP context, GPT-4V is given the following prompt: “This is a bimanual robot setup with a human standing on the other side of the table. Describe the human’s relative position on the table, the object the human is holding, the five bins, and their contents. The human has instructed the robot {instruction}. What could be the human’s intent with the object? ” The 31 choices are then listed and scored via MCQA. The {instruction} is chosen at random from the following list: 
    \begin{enumerate}
        \item put it in the bin
        \item put it with the others of the same color
        \item put it in the bin closest to me
        \item sort the object
        \item place it with the similar objects
        \item put it in any of the bins
        \item sort it with the left arm
        \item sort it with the right arm
    \end{enumerate}

\end{itemize}

\subsection{Additional Model Details}\label{Appendix: Models}

We now give additional predictor and policy details for each setting. We detail the architecture, training procedure, input representation, and output representation for each learned model.

\begin{table}
    \centering
    \begin{tabular}{|c|c|}
    \hline
       Transformer Encoder Layers  & 6 \\
       Transformer Decoder Layers  & 6 \\
       MLP Layers & 256 \\
       Hidden Dimension  & 256 \\
    \hline
    \end{tabular}
    \caption{Intent prediction architecture for Hallway and Social Navigation. }
    \label{tab: intent_predictor}
\end{table}

\begin{table}
    \centering
    \begin{tabular}{|c|c|}
    \hline
       Initial Learning rate & 1e-4 \\
       Learning rate decay epochs & $[50, 80, 90, 100]$\\
       Learning rate decay factor & 0.5\\
       Total Epochs & 200 \\
       Batch size & 64 \\
       Prediction Horizon & 100\\
       $\lambda$ & 1\\
    \hline
    \end{tabular}
    \caption{Intent prediction training details for Hallway and Social Navigation. }
    \label{tab: intent_predictor_training}
\end{table}

\begin{table}
    \centering
    \begin{tabular}{|c|c|}
    \hline
       Learning rate & 1e-4 \\
       Actor MLP Layers & 3 \\
       Critic MLP Layers & 3 \\
       Hidden Dimension & 64 \\
       Discounting factor $\gamma$ & 0.99\\
       Environments & 256 \\
       Batch size & 4096 \\
       Total Epochs & 800 \\
    \hline
    \end{tabular}
    \caption{PPO Training Parameters for learning the intent-conditioned policy in Hallway.}
    \label{tab: ppo_params}
\end{table}

\noindent
\textbf{Simulation: Hallway Navigation} 
\begin{itemize}
    \item[--] Intent Predictor: To predict the human car's intent, we train a transformer-based prediction model similar to \cite{salzmann2020trajectron++, lowe2017multi} to output a probability distribution over the hallway in addition to the future position of the human car. We encode the position histories for both agents using a three layer MLP with hidden dimension $256$. We process the encoded input using six transformer encoder layers and six transformer decoder layers, each with a hidden dimension of $256$. For the state prediction task, we predict with a time horizon of up to $100$ time steps. The full list of paramters are in Table \ref{tab: intent_predictor}.
    \item[--] Predictor training: To train the model, we use the following loss:
    \begin{equation}
        \mathcal L = \mathcal L_{CE} + \lambda \mathcal L_{MSE} 
    \end{equation}
    where $\mathcal L_{CE}$ is the cross entropy of the predicted intent distribution versus the true ground truth label, $\mathcal L_{MSE}$ is the mean square error from the human car's ground truth state, and $\lambda$ is a scalar that controls the relative weight of the state prediction task. We find the training to be relatively insensitive to $\lambda$, but that $\lambda=1$ works well in practice. We do not use knowledge of the human's static  intent in the prediction task.  The full set of training parameters are in Table \ref{tab: intent_predictor_training}.
    \item[--] Input representation: The predictor takes as input a matrix $X = (X_1, …, X_5) \in R^{T\times120}$ containing the concatenation of the observation histories for each possible intent. Each intent-conditioned observation $X_z\in R^{24}$ contains the joint position and heading history of both cars, their relative heading and position, the distance from each car to each hallway, and the distance from each car to their respective goals. 
    \item[--] Output representation: The predictor outputs log-softmax scores corresponding to each of the five hallways that the human could possibly take. 
    \item[--] Policy: We train an intent-conditioned PPO policy to maximize the forward progress jointly for both cars. To ensure satisfaction of the collision constraints, we terminate the episode for both cars if either car violates a collision constraint. We train a three-layer PPO policy using $256$ parallel environments and a hidden dimension of $64$, learning rate of $0.0001$, batch size of $4096$, and $32$ gradient steps per rollout. The full set of PPO training parameters are shown in Table \ref{tab: ppo_params}.

\end{itemize}

\noindent
\textbf{Simulation: Social Navigation} 
\begin{itemize}
    \item[--] Intent Predictor: We use an identical prediction architecture and training procedure as that of the Hallway Navigation task. Since the number of objects may change across the environments, we use the object locations as anchors for the transformer decoder similar to the strategy used by \cite{salzmann2020trajectron++}.
    \item[--] The training setup is identical to the Hallway Navigation task.
    \item[--] Input representation: The predictor takes as input a matrix $X = (X_1, …, X_{N_o}) \in R^{T\times350}$ containing the concatenation of the (zero-padded) observation histories for each possible intent. Each intent-conditioned observation $X_z \in R^{35}$ contains the joint position and heading history of both agents, their relative heading and position, the distance and heading from each human goal object. The maximum possible objects in a given scene $N_o$ is 10. 
    \item[--] Output representation: the predictor outputs log-softmax scores corresponding to each of the $N_o$ goal objects in the scene. 
    \item[--] Policy: To accommodate the close-quarters, constrained motion planning needed for social navigation, we use Habitat's \cite{puig2023habitat} built-in shortest-path navigation algorithm. During the contingency planning phase, the robot plans a shortest path for each of possible objects intended by the human. To ensure that the human and robot do not collide, we augment the robot's path planning dynamics with an additional artificial potential term that repels the robot from the human with force proportional to the inverse distance to the human. 
\end{itemize}

\noindent
\textbf{Hardware: Open-Category Sorting} 
\begin{itemize}
    \item[--] Intent Predictor: We use GPT-4V to select the intent using MCQA (labeling the bins as A: Blue Bin, B: White Bin, C: Gray Bin, or D: Not sure). Specifically, we obtain log-softmax scores from GPT-4V for the four MCQA options. We multiply the log-softmax scores by the optimal temperature obtained by RCIP before taking the softmax probabilities for each intent. 
    \item[--] Training: We do not perform any training or fine tuning for this setup. 
    \item[--] Input representation: The input to the VLM captioner is a set of four images: one image of the object to be sorted, and three images of the bins.
    \item[--] Output representation: The output of the prediction model is a set of softmax scores over the three bins (choices A, B, C), and an additional choice for the model not being sure (D). 
    \item[--] Policy: We use a pick-and-place primitive to grasp the object of interest and place it into the human’s intended bin. Since the policy is not learned, there is no training involved. 
\end{itemize}

\noindent
\textbf{Hardware: Bimanual Sorting} 
\begin{itemize}
    \item[--] Intent Predictor: We use GPT-4V to rank the possible intents. We then select the final intent using MCQA on the top 20 intents. Specifically, we obtain log-softmax scores from GPT-4V for the 20 MCQA options. We multiply the log-softmax scores by the optimal temperature obtained by RCIP before taking the softmax probabilities for each intent. 
    \item[--] Training: We do not perform any training or fine tuning for this setup. 
    \item[--] Input representation: The input to the VLM predictor is a single image of the scene. 
    \item[--] Output representation: The output of the prediction model is a set of softmax scores over the  top-5 intents from the 31 possible bin combinations. 
    \item[--] Policy: We use a pick-and-place primitive to hand over the object from the human and place it in one of the intended bins. Since the policy is not learned, there is no training involved. 
\end{itemize}

\begin{figure}
    \centering

    \subfigure[Temperature $=$ $2$ (KnowNo, before rescaling)]{\includegraphics[width=0.25\textwidth]{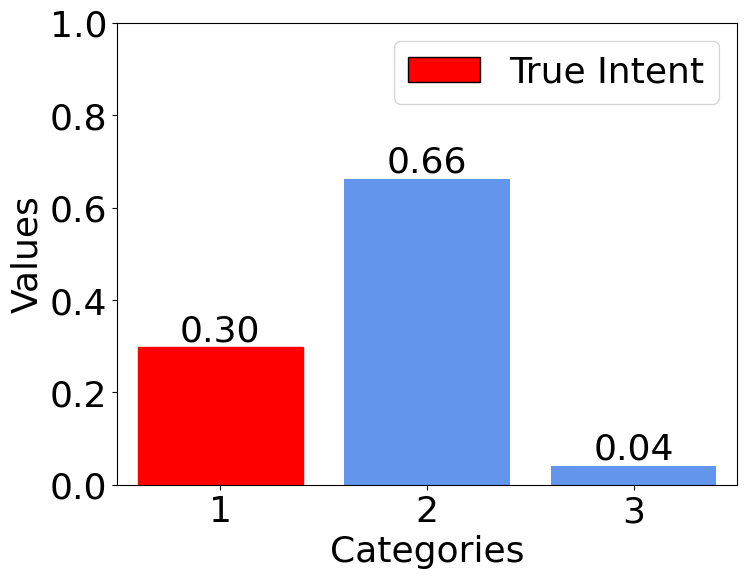}}
    \subfigure[Temperature $=$ $1$ (RCIP, after rescaling)]{\includegraphics[width=0.25\textwidth]{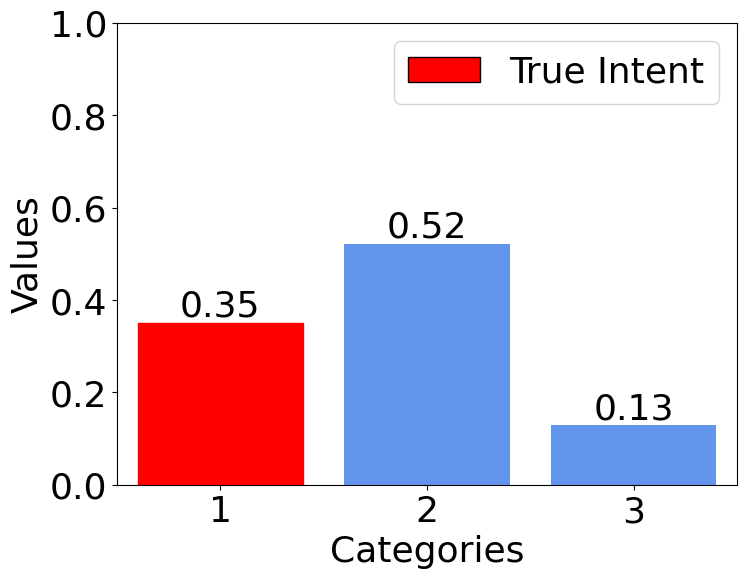}}
    \caption{Hypothetical predictor outputs for scenario A with intent space $\mathcal Z=\{1, 2, 3\}.$ Since the confidence for the true intent is less than the maximum confidence intent, the prediction is \textit{incorrect}. However, after temperature rescaling, the predictor is more confident in the true intent.  The ground-truth intent is shown in red. }
    \label{fig: temp_hypothetical_a}
\end{figure}

\begin{figure}
    \centering

    \subfigure[Temperature $=$ $2$ (before rescaling)]{\includegraphics[width=0.25\textwidth]{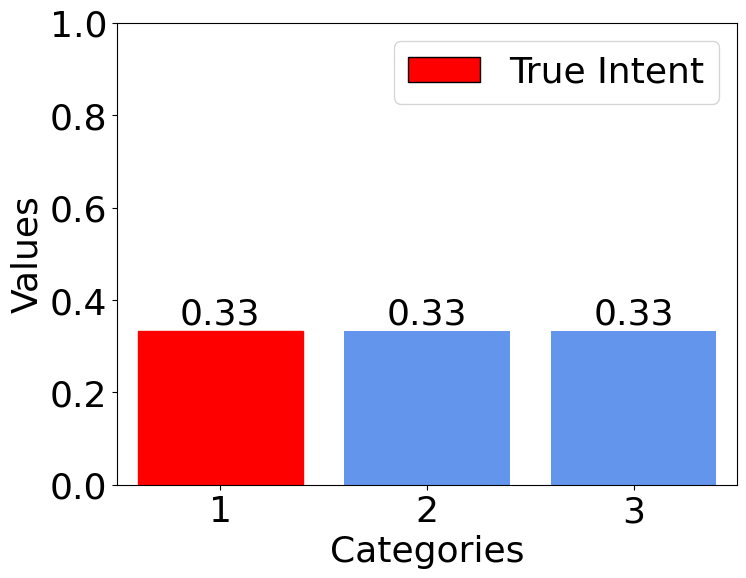}}
    \subfigure[Temperature $=$ $1$ (after rescaling)]{\includegraphics[width=0.25\textwidth]{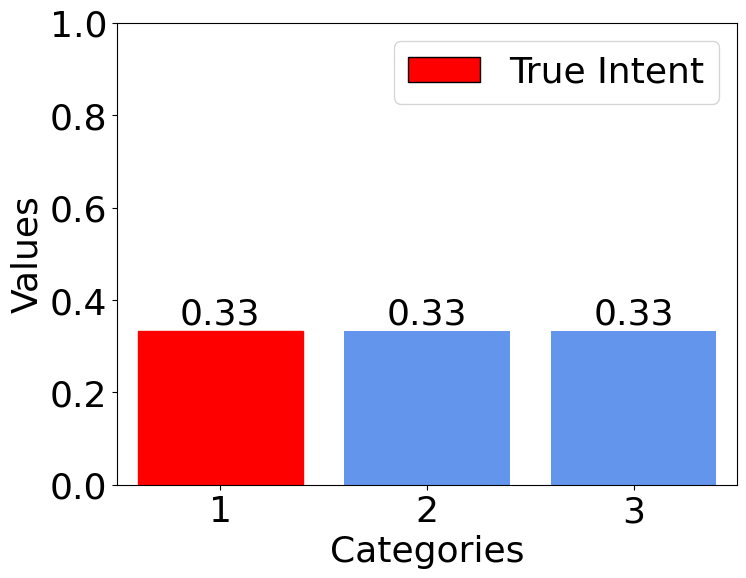}}
    \caption{Hypothetical predictor outputs for scenario B with intent space $\mathcal Z=\{1, 2, 3\}.$ The ground-truth intent is shown in red. Assume that the first intent is chosen in the event of a tie; thus, the prediction is \textit{correct}. Since the predictor is maximally \textit{uncertain}, temperature rescaling does not affect the confidence scores.}
    \label{fig: temp_hypothetical_b}
\end{figure}

\subsection{Additional Discussion} \label{Appendix: discussion}
\noindent
\textbf{Optimizing model temperature with RCIP enables better coverage than conformal prediction alone.} Temperature scaling is an effective ad-hoc calibration technique in neural network prediction \cite{platt1999probabilistic, guo2017calibration, xi2024does}. However, achieving optimal temperature calibration is usually a tedious process that requires hand-tuning of temperature. Instead, RCIP provides a lightweight procedure for achieving \textit{automatic} temperature calibration with performance guarantee on downstream risk. 

Temperature scaling can be used to increase the rate of singleton sets less-confident, but correct, predictors , resulting in a reduced rate of human intervention for a given coverage level. Conversely, temperature scaling can also be used to reduce the rate of incorrect prediction sets for overconfident predictors. To illustrate this point, consider the confidence scores for a distribution of only two scenarios: Scenario A in \ref{fig: temp_hypothetical_a} and Scenario B in \ref{fig: temp_hypothetical_b}. 

Now, consider the case where we want to ensure maximal coverage ($\alpha_{\text{cov}}=0$) with minimal help. To ensure coverage of scenario A, the robot needs to ask for help with a prediction threshold $\lambda \geq 0.30$ when the temperature is 2.  Thus, this results in a $100\%$ success rate, since the prediction for Scenario B is correct. However, Scenario B has a confidence in the ground-truth intent of $0.33$. Therefore, taking $\lambda=0.3$ forces the predictor to ask for help in Scenario B even when it is correct. The help rate for this temperature is $100\%$. Fig.~\ref{fig: non_conformity} (left) shows the non-conformity scores for this temperature, defined as $1-\texttt{confidence}(z_{\text{true}})$. To ensure coverage of the right-most (least confident) intents, the predictor has to ask for help for every intent to the left. 

After the temperature is rescaled to 1, the robot now has $0.35$ confidence in the true intent for Scenario A. Since this confidence now exceeds that of the true intent for Scenario B, we can set a prediction threshold of $\lambda \geq 0.35$. This ensures $100\%$ coverage with only $50\%$ help. In Fig.~\ref{fig: non_conformity} (right), the non-conformity score for Scenario A is now to the left of the non-conformity score for Scenario B, meaning that the predictor's incorrect predictions are those that it is least confident in.  To illustrate this phenomenon on a larger scale, we show a histograms for empty, singleton, and help-needed prediction sets for Hallway Navigation in Fig.~\ref{fig: hallway_rescaling}.

\begin{figure}
    \centering
    \subfigure[Temperature $=$ $2$ non-conformity scores (before rescaling)]{\includegraphics[width=0.25\textwidth]{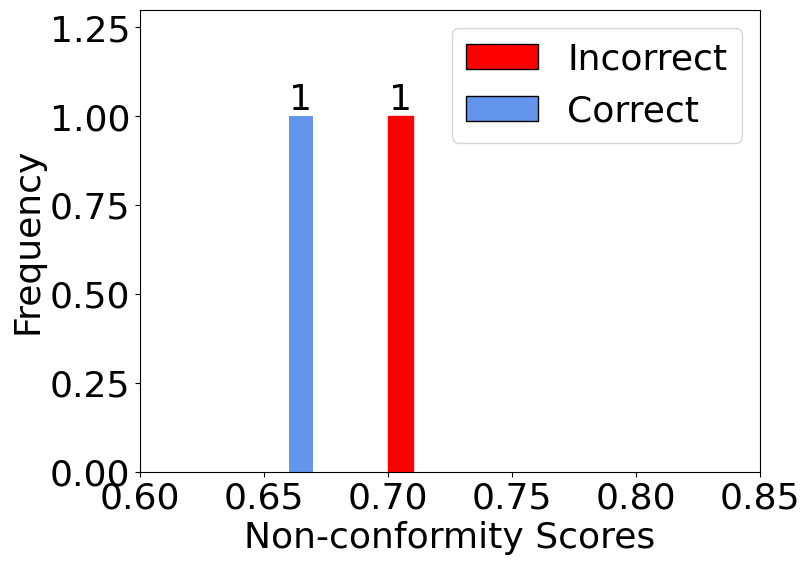}}
    \subfigure[Temperature $=$ $1$ non-conformity scores (after rescaling)]{\includegraphics[width=0.25\textwidth]{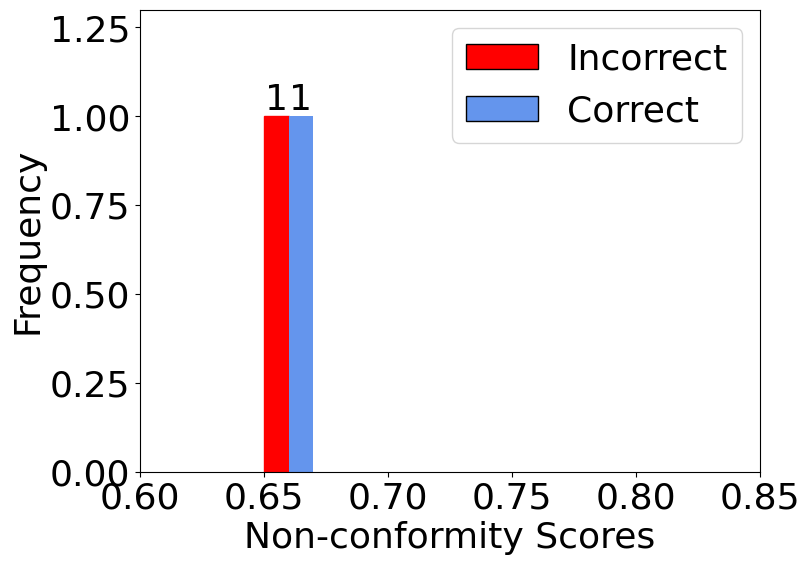}}
    \caption{Non-conformity scores before and after temperature rescaling.}
    \label{fig: non_conformity}
\end{figure}

\begin{figure}
    \centering
    \subfigure[Temperature $=$ $1$ non-conformity scores (before rescaling)]{\includegraphics[width=0.25\textwidth]{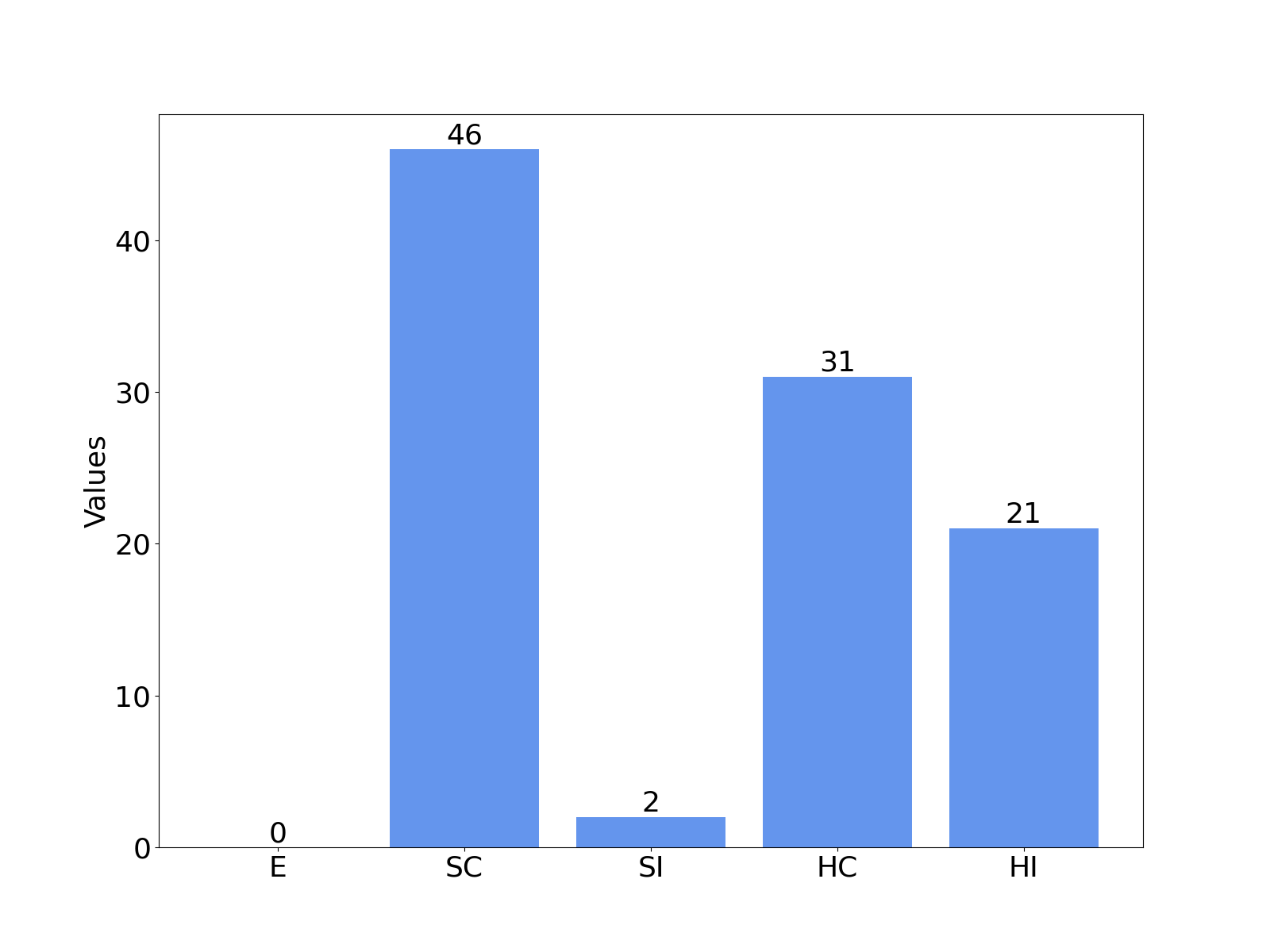}}
    \subfigure[Temperature $=$ $0.01$ non-conformity scores (after rescaling)]{\includegraphics[width=0.25\textwidth]{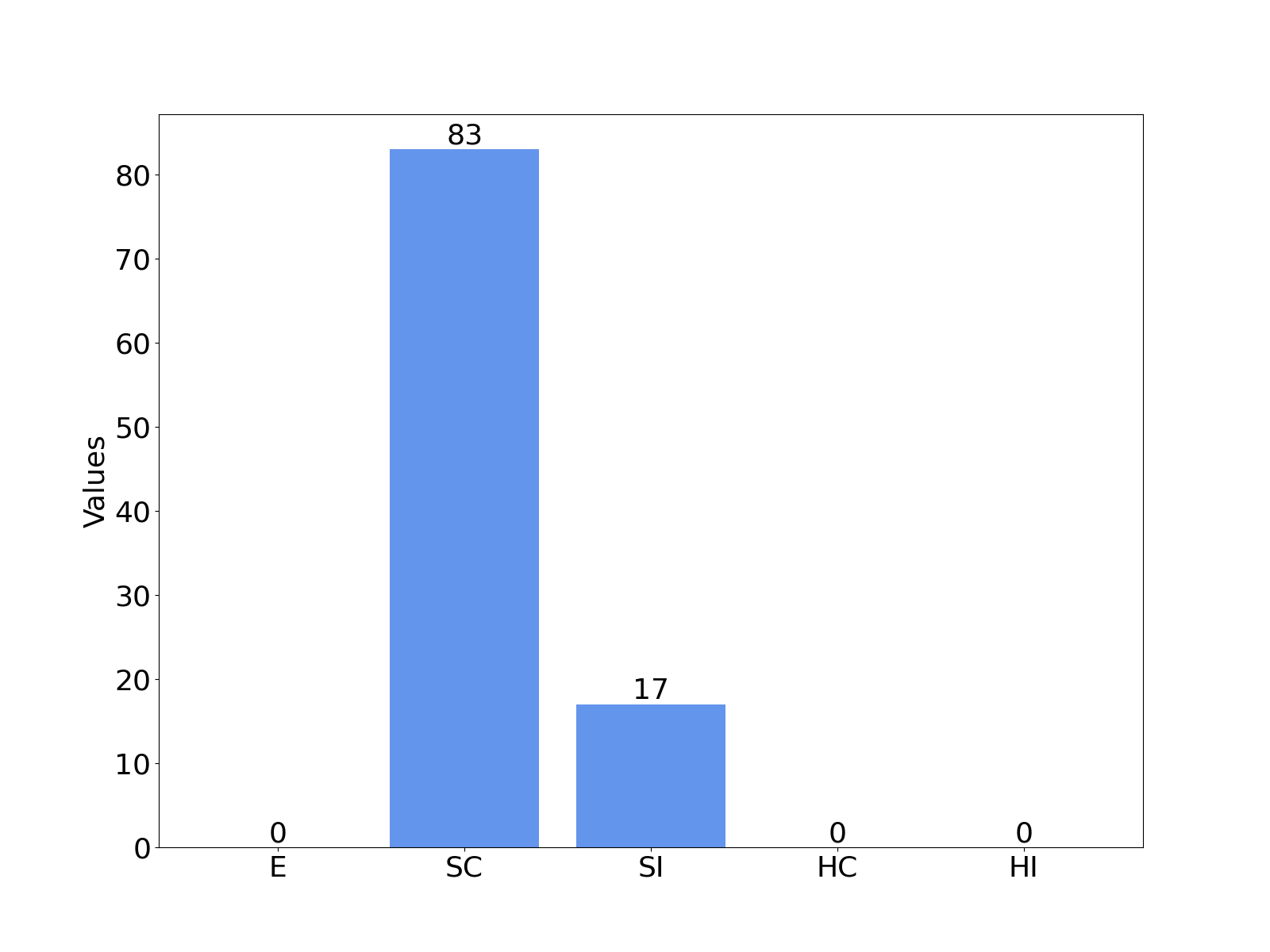}}
    \caption{Distribution of help rates and correctness of prediction sets on the test set of \textbf{Hallway Navigation} before and after rescaling. E = Empty, SC = Singleton (Correct), SI = Singleton (Incorrect), HC = Help (Correct), SI = Singleton (Incorrect).}
    \label{fig: hallway_rescaling}
\end{figure}

\noindent
\textbf{RCIP enables better coverage of intents than open-ended language planning.} In the Experiments (c.f. Section \ref{Experiments}), we assume that the KnowNo \cite{ren2023robots} baseline has access to the intent set. However, in the framework original proposed in \cite{ren2023robots}, the robot plans are generated by an LLM and need to be actuated using a language-conditioned policy such as RT-2\cite{brohan2023rt}, which may present challenges for downstream intent satisfaction. To illustrate this challenge, we show an example scenario from the Bimanual Manipulation experiment. The prompt image is show in Fig.~\ref{fig: conformal_input_vlm_planner} and the generated plans are shown in Fig.~\ref{fig: conformal_input_vlm_plans}. Several of the open-ended plans (marked in orange, red) are either ambiguous (resulting in the language-conditioned policy to confuse the human's intent, and causing the human's help to be ineffective) or fictitious (e.g., several descriptive objects proposed by GPT-4V do not exist). We envision that RCIP can be a useful alternative to open-ended language planning in situations where generated plans are not semantically meaningful. 

\begin{figure}
    \centering
\includegraphics[width=0.3\textwidth]{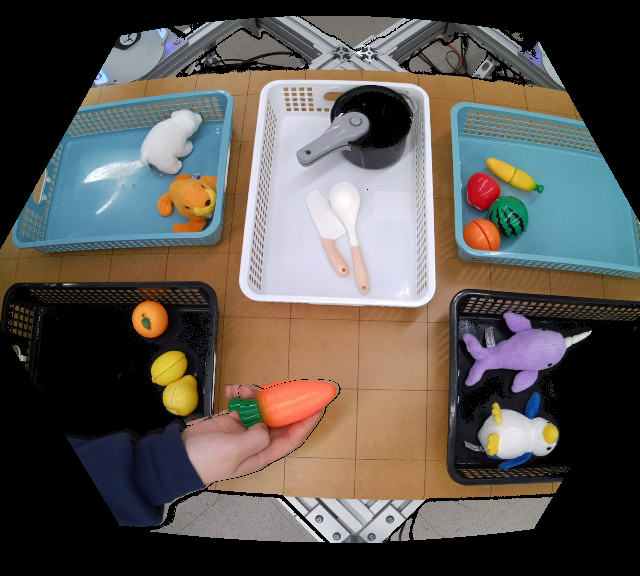}
    \caption{Example image prompt from the Bimanual Sorting scenario distribution.}
    \label{fig: conformal_input_vlm_planner}
\end{figure}

\begin{figure}
    \centering
    \includegraphics[width=0.9\textwidth]{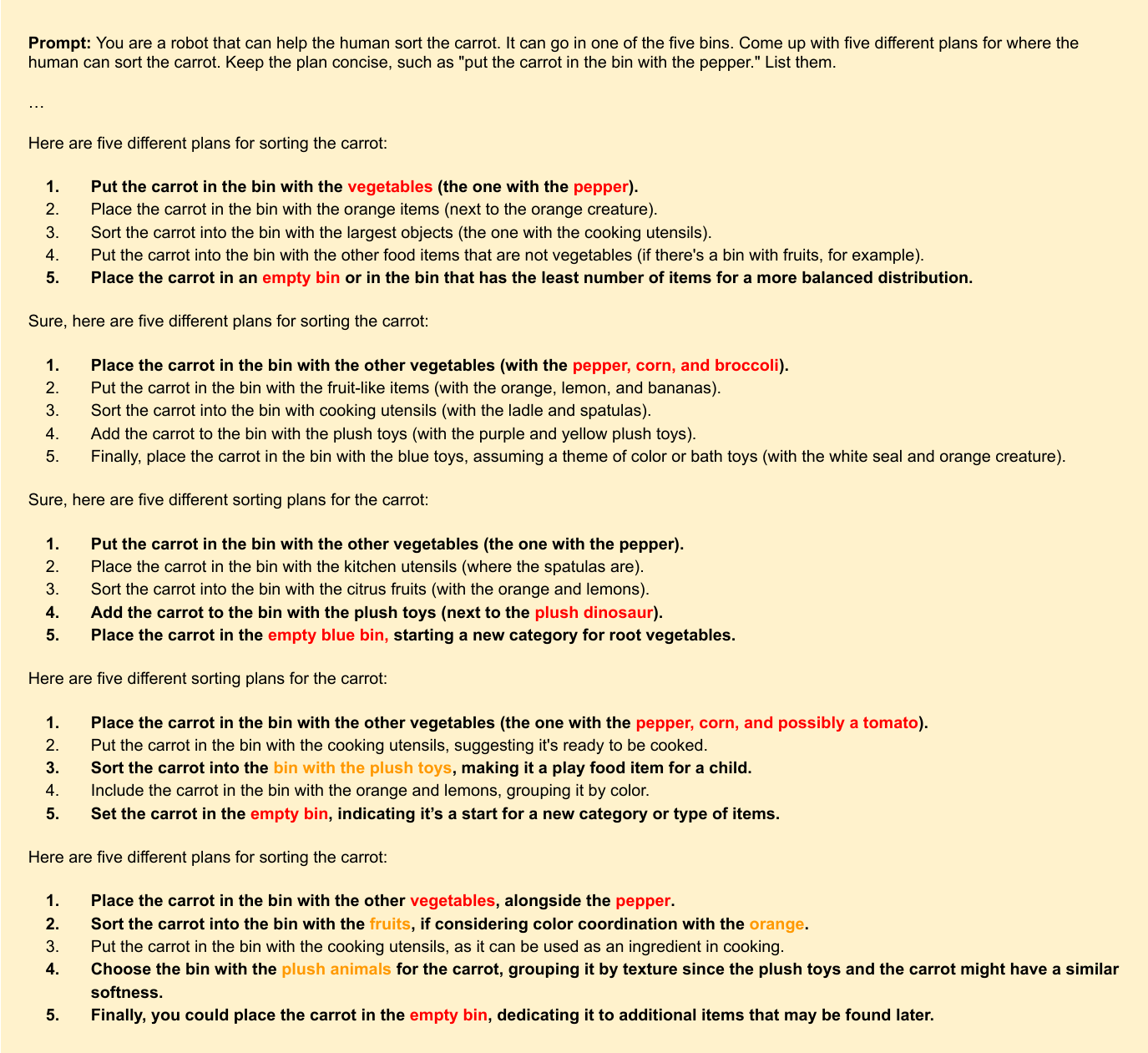}
    \caption{Five example open-ended language plans generated by GPT-4V for the example image shown in Fig.~\ref{fig: conformal_input_vlm_planner}. Non-existent objects are marked in \textcolor{red}{red}. Ambiguous objects are marked in \textcolor{orange}{orange}. Open-ended language plans used in the original KnowNo  \cite{ren2023robots} framework may pose a significant challenge for language-conditioned policies (e.g., with RT2\cite{brohan2023rt}) due to semantic mismatch.  RCIP avoids open-ended language planning by pre-specifying a fixed set of intents and using an intent-condtioned planner, ensuring consistency with the human's intent.}
    \label{fig: conformal_input_vlm_plans}
\end{figure}

\subsection{Additional Implementation Details}\label{Appendix: implementation}

\begin{figure}
    \centering
    \includegraphics[width=0.3\textwidth]{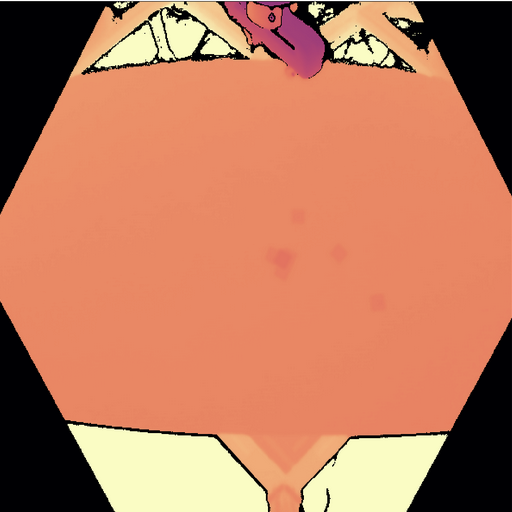}
    \includegraphics[width=0.3\textwidth]{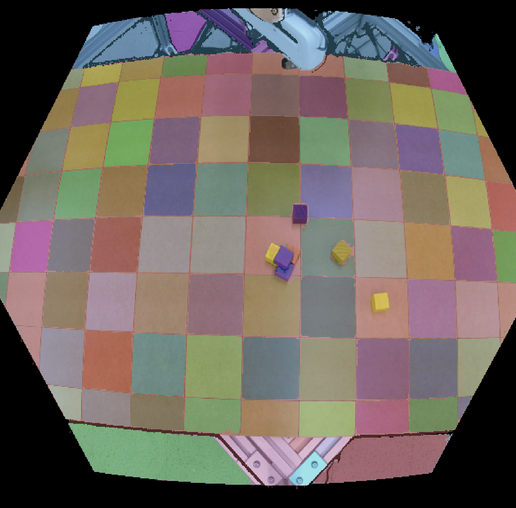}
    \includegraphics[width=0.3\textwidth]{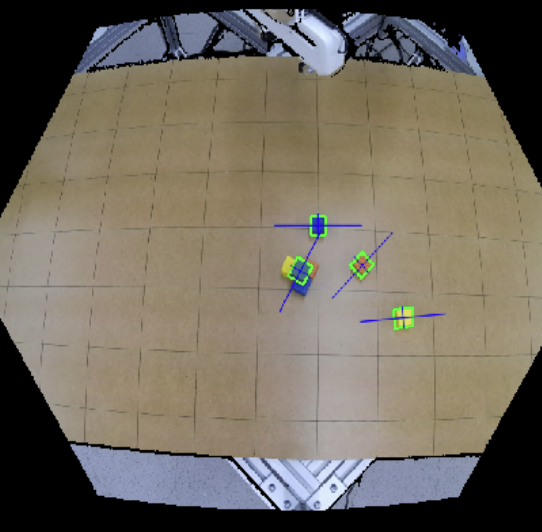}
    \caption{Perception pipeline for manipulation tasks.}
    \label{fig:enter-label}
\end{figure}

\noindent
\textbf{Perception.} In the Hallway and Habitat experiment, we assume that the robot has full observability of the human's position. In the Open Category Sorting Experiment, we apply several filters to extract the pose of the target for sorting. We first apply create a depth mask using the RGB-D image to filter out foreground (e.g., the Franka Panda end effector) and background (e.g., the floor). For segmentation, we use the Segment Anything \cite{kirillov2023segment} model with uniform sampling over the image to obtain masks for any potentially relevant objects. Using the depth filter, masks for irrelevant objects (e.g., each grid square on the table) are removed. For each Segment Anything mask, a rectangle is fitted to the contour that defines the mask's boundary. For each object, the position and yaw in the robot's coordinate frame is determined using the centroid and rotation for the rectangle approximation. The bimanual task does not require precise object localization due to the transfer between the human and the robot, so the Azure Kinect's RGB image is instead passed directly to GPT-4V.

\noindent
\textbf{Low-level control.} For the Hallway and Habitat setups, we directly learn a low-level policy using kinematic state information. The Hallway car follows a 4D kinematic bicycle model with a two dimensional control input corresponding to target velocity and steering angle. The Habitat spot robot follows a 2D kinematic unicycle model with direct velocity and steering angle commands. For the sorting and bimanual experiments, we solve trajectories for pick-and-place maneuvers using inverse kinematics. Since the bimanual task requires object transfer between the human and robot at variable heights, we use simple scripted primatives to take the object from the human at various positions around the table (left, center, and middle) and drop the object in each of the five bins. 

\noindent
\textbf{Human Feedback.} In RCIP, we directly prompt the human for their intent by giving a numbered list for each intent. The robot picks the low-level action conditioned on this intent, allowing the human to give high-level information while the robot distinguishes between low-level action. Since intents can be many-to-one, clarification in the intent space is equivalent to clarification in the action space as long as the robot can compute the action corresponding to the human's true intent. 
